\documentclass[reprint,a4paper,twoside,11pt]{iopart}

\usepackage{graphicx}
\usepackage{diagbox} 
\usepackage{hyperref}
\expandafter\let\csname eqalign\endcsname\relax
\expandafter\let\csname equation*\endcsname\relax
\expandafter\let\csname endequation*\endcsname\relax
\usepackage{amsmath}
\usepackage{amsfonts}
\usepackage{amssymb} 
\usepackage{amsthm}
\usepackage{mathtools,nccmath}
\usepackage{txfonts}
\usepackage{url}
\usepackage{iopams}
\usepackage{dsfont}
\usepackage{nicefrac}       
\usepackage{xfrac}
\usepackage[USenglish]{babel}
\usepackage{multirow}
\usepackage{paralist}
\usepackage{booktabs}
\usepackage{microtype}      
\usepackage{xcolor,colortbl}
\usepackage{hhline}
\usepackage{bbm}
\usepackage{tabularx}
\usepackage{lscape}
\usepackage{multirow,makecell}
\usepackage{color,soul}
\usepackage{algorithm}
\usepackage{algorithmicx}
\usepackage{algpseudocode}
\usepackage{xspace}
\usepackage{wrapfig}
\usepackage{harpoon}
\usepackage{mathrsfs}
\usepackage{bm}

\def\e{\mbox{e}}

\def\S{{\bf S}}

\def\d{\text{d}}
\def\e{\text{e}}
\def\E{\mathbb{E}}
\usepackage[caption=false]{subfig}

\usepackage{pgfplots}

\newtheorem{theorem}{Theorem}[section]
\newtheorem{example}{Example}[section]
\newtheorem{corollary}{Corollary}[section]
\newtheorem{assumption}{Assumption}[section]
\newtheorem{definition}{Definition}[section]

\usepackage[draft]{changes}

\begin{document}
\title[Wasserstein-distance approach for reconstructing jump-diffusion
  processes]{An efficient Wasserstein-distance approach for reconstructing
  jump-diffusion processes using parameterized neural networks}

\author{Mingtao Xia$^{1}$\footnote{corresponding author}, Xiangting Li$^2$, Qijing Shen$^3$, Tom Chou$^4$}

\address{{\small $^1$ Courant Institute of Mathematical Sciences, New York University, New York, NY
10012, USA}\\ {\small $^2$ Department of Computational Medicine, UCLA, Los Angeles, CA
90095, USA}\\ {\small $^3$ Nuffield Department of Medicine,
    University of Oxford, Oxford OX2 6HW, UK}\\
  {\small $^4$ Department of Mathematics, UCLA, Los Angeles, CA
90095, USA}}
\ead{xiamingtao@nyu.edu, xiangting.li@ucla.edu,qijing.shen@ndm.ox.ac.uk, tomchou@ucla.edu}
\date{\today}

\renewcommand{\thefootnote}{\fnsymbol{footnote}}





\begin{abstract}
We analyze the Wasserstein distance ($W$-distance) between two
probability distributions associated with two multidimensional
jump-diffusion processes. Specifically, we analyze a temporally
decoupled squared $W_2$-distance, which provides both upper and lower
bounds associated with the discrepancies in the drift, diffusion, and
jump amplitude functions between the two jump-diffusion processes.
Then, we propose a temporally decoupled squared $W_2$-distance method
for efficiently reconstructing unknown jump-diffusion processes from
data using parameterized neural networks. We further show its
performance can be enhanced by utilizing prior information on the
drift function of the jump-diffusion process.  The effectiveness of
our proposed reconstruction method is demonstrated across several
examples and applications.


\end{abstract}


\section{Introduction}
Jump-diffusion processes are widely used across many disciplines such
as finance \cite{merton1976option, jang2007jump, maekawa2008jump},
biology \cite{gao2022data}, epidemiology \cite{tesfay2021dynamics},
and so on. A $d$-dimensional jump-diffusion process may be written in
the following form \cite{mehri2019stochastic}:
\begin{equation}
    \d \bm{X}(t) = \bm{f}(\bm{X}(t), t)\d t + \bm{\sigma}(\bm{X}(t),
    t)\d \bm{B}_t + \int_U\bm{\beta}(\bm{X}(t), \xi, t)\tilde{N}( \d
    t, \nu(\d\xi))
    \label{model_equation}
\end{equation}
%
Here, $\bm{X}(t)\in\mathbb{R}^d$ is a $d$-dimensional jump-diffusion
process and $\bm{B}_t\coloneqq(B_{1, t},...,B_{m, t})$ is a standard
$m$-dimensional white noise; $\tilde{N}$ is a compensated Poisson
process of intensity $\nu(\d\xi) \d t$ independent of $\bm{B}_t$:
\begin{equation}
    \tilde{N}(\d t, \nu(\d\xi)) \coloneqq N(\nu(\d\xi), \d t) - \nu(\d\xi) \d t,
\end{equation} 
where $N(\nu(\d\xi), \d t)$ is a Poisson process with intensity
$\nu(\d\xi) \d t$ and $\nu(\d\xi)$ is a measure defined on
$U\subseteq\mathbb{R}$, the measure space of the Poisson
process. $N(A, B)$ and $N(C, D)$ are independent if $(A\times B)\cap
(C\times D)=\emptyset, A,C\subseteq \mathcal{B}(U)$ and $B, D\subseteq
\mathcal{B}([0, T])$. $\mathcal{B}(U)$ and $\mathcal{B}([0, T])$
denote the $\sigma$-algebra associated with $U$ and $[0, T]$,
respectively.  The drift, diffusion, and jump functions are defined by
\begin{equation}
\begin{aligned}
\bm{f}\coloneqq & \big(f_1(\bm{X}, t),...,f_d(\bm{X}, t)\big)
  \in C(\mathbb{R}^d\times[0, T], \mathbb{R}^d), \\
  \bm{\sigma}\coloneqq & \big(\sigma_{i, j}(\bm{X}, t)\big)
  \in C(\mathbb{R}^d\times[0, T], \mathbb{R}^{d\times m}),\\
  \bm{\beta}\coloneqq & \big(\beta_{i}(\bm{X}, \xi, t)\big)
  \in C(\mathbb{R}^d\times U\times[0, T], \mathbb{R}^d),
\end{aligned}
\end{equation}
respectively. Specifically, if $U=\{1,...,n\}$, then
Eq.~\eqref{model_equation} becomes
\begin{equation}
  \d X_i(t) = f_i(\bm{X}(t), t)\d t
  + \sum_{j=1}^m \sigma_{i, j}(\bm{X}(t), t)\d (B_j)_t
  + \sum_{s=1}^n\beta_{i, s}(\bm{X}(t), s, t)\tilde{N}_{s}(\d t, \nu(s))
    \label{model_equation_discrete}
\end{equation}
for $i=1,..,d$. Here, each $\tilde{N}_s$ is a compensated Poisson
process with intensity $\nu(s)\d t$. $\tilde{N}_{s_1}$ and
$\tilde{N}_{s_2}$ are independent if $s_1\neq s_2$. When
$\bm{\beta}\equiv 0$, Eq.~\eqref{model_equation} reduces to the pure
diffusion process.

In this paper, we study the problem of reconstructing a jump-diffusion
process Eq.~\eqref{model_equation} or Eq.~\eqref{model_equation_discrete}
from observed data $\bm{X}(t)$ at different time points by using a
different jump-diffusion process

\begin{eqnarray}
\d \hat{\bm{X}}(t) = \hat{\bm{f}}(\bm{X}(t), t)\d t +
\hat{\bm{\sigma}}(\hat{\bm{X}}(t), t)\d \hat{\bm{B}}_t +
\!\int_U\hat{\bm{\beta}}(\hat{\bm{X}}(t), \xi, t)\hat{N}\big(\d t,
\nu(\d\xi)\big)
\label{approximate_equation}
\end{eqnarray}
to approximate Eq.~\eqref{model_equation}. In
Eq.~\eqref{approximate_equation}, $\hat{\bm{B}}_t$ is a
$m$-dimensional standard Brownian motion that is independent of
$\bm{B}_t$ and $\tilde{N}$ in Eq.~\eqref{model_equation}; $\hat{N}(\d
t, \nu(\d\xi))$ is a compensated Poisson process of intensity $\d
\nu(\xi) \d t$ and independent of $\bm{B}_t$, $\tilde{N}$ in
Eq.~\eqref{model_equation} as well as $\hat{\bm{B}}_t$.  Specifically,
we are interested in reconstructing the jump-diffusion process
Eq.~\eqref{model_equation} with little or no prior information on the
drift, diffusion, and jump functions $\bm{f}, \bm{\sigma}$, and
$\bm{\beta}$. To reconstruct or approximate Eq.~\eqref{model_equation}
using Eq.~\eqref{approximate_equation}, we wish to find small errors
in the drift, diffusion, and jump functions, \textit{i.e.}, to find
$\hat{\bm{f}}, \hat{\bm{\sigma}}$, and $\hat{\bm{\beta}}$ such that
$\bm{f}-\hat{\bm{f}}$, $\bm{\sigma}-\hat{\bm{\sigma}}$, and
$\bm{\beta}-\hat{\bm{\beta}}$ are small.

Thus far, most studies related to jump-diffusion processes have
focused on the forward-type problem of efficient simulation of a
jump-diffusion process given coefficients \cite{casella2011exact,
  metwally2002using}. There are also several studies on the
statistical properties of jump-diffusion processes such as their first
passage times \cite{kou2003first,zhang2009first}. While there has been
some research into the inverse problem of reconstructing a general
pure diffusion process, there has been little work on reconstructing
unknown jump-diffusion processes from sample trajectories although two
main strategies have been proposed.  First, regression methods are
applied to determine unknown parameters if the forms of drift,
diffusion, and jump functions ($\bm{f}$, $\bm{\sigma}$ and
$\bm{\beta}$ in Eq.~\eqref{model_equation}) are known. Unknown
parameters in these functions can then be determined from data
\cite{gorjao2019analysis,ramezani1998maximum}. Another strategy for
reconstructing a jump-diffusion process is to calculate the empirical
probability density function $p(\bm{X}, t)$ from observation data
$\bm{X}(t)$ and then reconstruct the integrodifferential equation
satisfied by $p(\bm{X}, t)$ \cite{gorjao2023jumpdiff}. Yet, this
method requires a large number of observations at different time
points to obtain a good empirical approximation of the density
function $p(\bm{X}, t)$.  Recently, a
Wasserstein-generative-adversarial-network(WGAN)-based method was
proposed to reconstruct the jump-diffusion process
Eq.~\eqref{model_equation} \cite{gao2022data}. However, training a
WGAN can be intricate and computationally expensive.

Recent advancements in machine learning make it possible to use
parameterized neural networks (NNs) for representing $\hat{\bm{f}},
\hat{\bm{\sigma}}$, and $\hat{\bm{\beta}}$ in
Eq.~\eqref{approximate_equation} which approximate $\bm{f},
\bm{\sigma}$, and $\bm{\beta}$ in Eq.~\eqref{model_equation}. For
example, a recent \texttt{torchsde} package in Python
\cite{li2020scalable} models pure diffusion processes (SDEs with
Brownian noise) by using parameterized neural networks.  These methods
have been used in the reconstruction of diffusion processes. For
example, in \cite{tzen2019neural}, a deep Gaussian latent model has
been applied for reconstructing a pure diffusion process;
\cite{tong2022learning} uses the neural SDE model to reconstruct a
stochastic differential equation with Brownian noise by minimizing a
KL-divergence-based loss function.  In
\cite{kidger2021neural,chen2023learning}, generative adversarial
networks were used to reconstruct general stochastic differential
equations including a Brownian motion noise term without requiring
prior knowledge of the specific forms of the drift or diffusion
functions.

Another recent work analyzes the upper bound for a smooth Wasserstein
distance between two distributions associated with two 1D
jump-diffusion processes at a given time \cite{breton2024wasserstein}.
Since the $W$-distance can effectively measure discrepancies between
probability measures over a metric space
\cite{villani2009optimal,zheng2020nonparametric}, an efficient
squared-Wasserstein-distance-based method for reconstructing pure
diffusion processes from data, without the need to specify forms for
the drift and diffusion, was recently proposed \cite{xia2023a}.
General jump-diffusion processes are distinct from pure diffusion
processes because the trajectories of jump-diffusion processes are
discontinuous.  Thus, it remains unclear whether the Wasserstein
distance can also be employed to reconstruct an unknown jump-diffusion
process from data.

\subsection{Contribution}
In this paper, we analyze the $W$-distance between two probability
distributions associated with two multidimensional jump-diffusion
processes Eqs.~\eqref{model_equation} and
~\eqref{approximate_equation}.  We then show that a temporally
decoupled squared Wasserstein distance can serve as effective
\textbf{upper and lower error bounds} on the errors in the drift,
diffusion, and jump functions $\bm{f}-\hat{\bm{f}}$,
$\bm{\sigma}-\hat{\bm{\sigma}}$, and $\bm{\beta}-\hat{\bm{\beta}}$ in
Eqs.~\eqref{model_equation} and~\eqref{approximate_equation},
respectively.  This temporally decoupled squared Wasserstein distance
can be effectively evaluated using finite-sample observations at
discrete time points. Thus, we propose using this temporally decoupled
squared $W_2$-distance to reconstruct general jump-diffusion processes
with the help of parameterized neural networks. Furthermore, we
explore how prior information on the drift function enhances the
performance of our temporally decoupled squared Wasserstein distance
method.  Our results greatly extend the results in \cite{xia2023a}
(reconstructing 1D pure diffusion processes) to allow for the
reconstruction of multidimensional jump-diffusion
processes. Specifically, we

\begin{enumerate}
\item[1.] prove that the $W$-distance is a lower bound for the errors
  $\bm{f}-\hat{\bm{f}}, \bm{\sigma}-\hat{\bm{\sigma}}$, and
  $\bm{\beta}-\hat{\bm{\beta}}$. Thus, minimizing the $W$-distance is
  necessary for reconstructing the multidimensional jump-diffusion
  process Eq.~\eqref{model_equation}.
\item[2.] analyze a temporally decoupled squared $W$-distance defined
  in \cite{xia2023a} and show that it can be efficiently evaluated by
  finite-sample empirical distributions. Thus, it is suitable to serve
  as a loss function to minimize for reconstructing
  Eq.~\eqref{model_equation} using parameterized neural networks.
\item[3.] conduct numerical experiments to demonstrate the efficacy of
  using the temporally decoupled squared Wasserstein distance to
  reconstruct jump-diffusion processes. Our temporally decoupled
  squared Wasserstein method performs better than some other benchmark
  methods. Additionally, we propose incorporating prior information on
  the drift function, which greatly improves the accuracy of
  reconstructed diffusion and jump functions.
\end{enumerate}

\subsection{Organization}
In Section~\ref{section2}, we analyze how the $W$-distance between the
probability measures associated with solutions to two jump-diffusion
processes Eqs.~\eqref{model_equation} and \eqref{approximate_equation}
can be a lower bound of the errors in the reconstructed drift,
diffusion, and jump functions $\bm{f}-\hat{\bm{f}}$,
$\bm{\sigma}-\hat{\bm{\sigma}}$, and $\bm{\beta}-\hat{\bm{\beta}}$.
In Section \ref{section3}, we analyze a temporally decoupled squared
$W_2$-distance and show how it can be more effectively evaluated than
the squared $W_2$ distance. Specifically, the temporally decoupled
squared $W_2$ distance is smaller than the $W_2$ distance analyzed in
Section~\ref{section2} while providing an upper bound of the errors in
the reconstructed drift, diffusion, and jump functions. Thus,
Sections~\ref{section2} and ~\ref{section3} together show that our
temporally decoupled squared $W_2$ distance provides both upper and
lower error bounds. In Section \ref{section4}, numerical experiments
are carried out to compare our proposed jump-diffusion process
reconstruction methods with other methods for reconstructing different
jump-diffusion processes. Additionally, we show how prior information
on the drift function of the ground truth jump-diffusion process
Eq.~\eqref{model_equation} improves the reconstruction of the
diffusion and jump functions in Eq.~\eqref{model_equation}.  In
Section~\ref{summary}, we summarize our proposed jump-diffusion
process reconstruction approach and suggest potential future
directions.

\section{The $W$-distance between the probability measures associated with
 the jump-diffusion processes in Eqs.~\eqref{model_equation} and~\eqref{approximate_equation}}
\label{section2}
In this section, we shall show how the $W$-distance between the
probability measures associated with the two jump-diffusion processes
Eqs.~\eqref{model_equation} and~\eqref{approximate_equation} can serve
as a lower bound for the errors $\bm{f}-\hat{\bm{f}}$,
$\bm{\sigma}-\hat{\bm{\sigma}}$, and $\bm{\beta}-\hat{\bm{\beta}}$.
First, we specify the assumptions on the jump-diffusion processes in
Eqs.~\eqref{model_equation} and~\eqref{approximate_equation}.

\begin{assumption}
\label{assumptions}
\rm We assume that the jump-diffusion processes defined in
Eq.~\eqref{model_equation} 
satisfy the
following conditions:

\begin{enumerate}
\item For each non-increasing sequence $A_i\subseteq U$ converging to
  the empty set $\emptyset$, $\E\big[|\tilde{N}(t,
    A)|^2\big]\rightarrow 0, \forall t\geq 0$.
\item $\tilde{N}(t, A)$ is a c$\text{\`a}$dl$\text{\`a}$g martingale
  for all $A\subseteq U, t>0$, and $\E[|\tilde{N}(t,U)|^2]<\infty$.
\item $\tilde{N}\big(\d t, \nu(\d\xi)\big)$ is an orthogonal
  martingale measure with intensity $\d t\cdot\nu(\d\xi)$,
  \textit{i.e.}, for any $A, B\subseteq U$ and $t_1<t_2, t_3\leq t_4$
  and any $\beta_1(\xi, t), \beta_2(\xi, t)\in L^2([0, T]\times U)$
  (the measure on $U$ is $\nu$), we have
  \begin{eqnarray}
    \hspace{-1cm}\E\bigg[\!\medint\int_{t_1}^{t_2}\!\!\!
      \medint\int_A \beta_1(\xi, t)
        \tilde{N}\big(\d t, \nu(\d\xi)\big)
        \medint\int_{t_3}^{t_4}\!\!\!\medint\int_B
        \beta_2(\xi, t) \tilde{N}\big(\d t,
        \nu(\d\xi)\big)\bigg] = \!\!\!\!\!\medint\int\limits_{[t_1, t_2]\cap [t_3,
          t_4]}\!\!\!\!\medint\int_{A\cap B}\beta_1(\xi, t)\beta_2(
      \xi, t)\nu(\text{d}\xi)\d{t}.
  \end{eqnarray}
\item Trajectories generated from both jump-diffusion processes,
  Eqs.~\eqref{model_equation} and \eqref{approximate_equation}, reside
  in the space $L^2([0, T], \mathbb{R}^d)$.
\item The drift, diffusion, and jump functions are all uniformly
  Lipschtiz continuous, $\textit{i.e.}$, there exists three positive
  constants $\overline{f}, \overline{\sigma}, \overline{\beta}<\infty$
  such that $\forall \bm{X}^1=(X^1_1,...,X^1_d),\forall
  \bm{X}^2=(X^2_1,...,X^2_d)\in\mathbb{R}^d$,

\begin{equation}
  \begin{aligned}
  \big|f_i(\bm{X}^1, t) - f_i(\bm{X}^2, t)\big| &
  \leq \frac{\overline{f}}{d}\sum_{i=1}^d \big|X_i^1-X_i^2\big|, \,\,\,\forall i=1,...,d \\
  \big|\sigma_{i, j}(\bm{X}^1, t) - \sigma_{i, j}(\bm{X}^2, t)\big| & \leq
  \frac{\overline{\sigma}}{d}\sum_{i=1}^d \big|X_i^1-X_i^2\big|,
  \,\,\forall i=1,...,d, \,\,\forall j=1,...,m, \\
  \big|\beta_i(\bm{X}^1, \xi, t) - \beta_i(\bm{X}^2, \xi, t)\big| & \leq
  \frac{\overline{\beta}}{d}\sum_{i=1}^d \big|X_i^1-X_i^2\big|, \,\,\,\forall i=1,...,d.
  \end{aligned}
\label{lipschtiz}
\end{equation}
\end{enumerate}
Furthermore, we assume that conditions (i)-(iv) also hold for the
compensated Poisson process $\hat{N}$ in
Eq.~\eqref{approximate_equation}, and that condition (v) holds for the
drift, diffusion, and jump functions in
Eq.~\eqref{approximate_equation}.
\end{assumption}

Now consider the $W$-distance between the distributions associated
with solutions generated from the target jump-diffusion process
Eq.~\eqref{model_equation} and the approximate jump-diffusion process
Eq.~\eqref{approximate_equation}, as defined below.

\begin{definition} 
\rm 
\label{def:W2}
For two $d$-dimensional jump-diffusion processes

\begin{equation}
\bm{X}(t)=(X_1(t),...,X_d(t)), \,\,\,\hat{\bm{X}}(t)
= (\hat{{X}}_1(t),...,\hat{X}_d(t)),\,\, t\in[0, T],
\label{sde_dimension}
\end{equation}
in the separable space $\big(L^2([0, T]; \mathbb{R}^d),
\|\cdot\|\big)$ with two associated probability distributions $\mu,
\hat{\mu}$, respectively, the $W_p$-distance $W_p(\mu, \hat{\mu})$ for
$1\leq p\leq 2$ is defined as
\begin{equation}
W_p(\mu, \hat{\mu}) \coloneqq \inf_{\pi(\mu, \hat \mu)}
\E_{(\bm{X}, \hat{\bm{X}})\sim \pi(\mu, \hat \mu)}\big[\|{\bm{X}}
  - \hat{{\bm{X}}}\|^{p}\big]^{\frac{1}{p}}.
\label{pidef}
\end{equation}
In Eq.~\eqref{pidef}, the norm $\|\cdot\|$ is defined as
$\|\bm{X}\|\coloneqq \Big(\int_0^T \sum_{i=1}^d |X_i(t)|^2\d
t\Big)^{\frac{1}{2}}$ and $\pi(\mu, \hat \mu)$ iterates over all
\textit{coupled} distributions of $\bm{X}(t), \hat{\bm{X}}(t)$,
defined by the condition

\begin{equation}
\begin{cases}
  {\bm{P}}_{\pi(\mu, \hat \mu)}\left(A \times L^2([0, T];
  \mathbb{R}^d)\right) ={\bm{P}}_{\mu}(A),\\
         {\bm{P}}_{\pi(\mu, \hat \mu)}\left(L^2([0, T]; \mathbb{R}^d)\times A\right)
         = {\bm{P}}_{\hat \mu}(A), 
\end{cases}\forall A\in \mathcal{B}\Big(L^2([0, T]; \mathbb{R}^d)\Big),
\label{pi_def}
\end{equation}
where $\mathcal{B}\Big(L^2([0, T]; \mathbb{R}^d)\Big)$ denotes the
Borel $\sigma$-algebra associated with the space of $d$-dimensional
functions in $L^2([0, T]; \mathbb{R}^d)$.
\end{definition}

To prove that $W_p(\mu, \hat{\mu})$ defined in Eq.~\eqref{pidef} is a
lower bound for the errors in the drift, diffusion, and jump functions
$\bm{f}-\hat{\bm{f}}$, $\bm{\sigma}-\hat{\bm{\sigma}}$, and
$\bm{\beta}-\hat{\bm{\beta}}$, we first prove the following theorem.

\begin{theorem}
\rm 
\label{theorem1}
     Suppose $\bm{X}(t)$ and $\hat{\bm{X}}(t)$ are two $d$-dimensional
     jump-diffusion processes that are determined by
     Eq.~\eqref{model_equation} and Eq.~\eqref{approximate_equation}. We
     denote
\begin{equation}
\d \tilde{\bm{X}}(t) = \hat{\bm{f}}\big(\tilde{\bm{X}}(t), t\big)
+ \hat{\bm{\sigma}}\big(\tilde{\bm{X}}(t), t\big)\d \bm{B}_t
+\! \int_U\hat{\bm{\beta}}
\big(\tilde{\bm{X}}(t), \xi, t\big)\tilde{N}\big(\d t, \nu(\d\xi)\big)
\label{tilde_equation}
\end{equation}
and assume that

\begin{equation}
\int_0^t \big( X_i(s^-) - \tilde{X}_i(s^-)\big)
\big(\sigma_{i, j} - \hat{\sigma}_{i, j}\big)\d B_{j, t},\,\,\,\int_0^t
\!\int_U \big(X_i(s^-) - \tilde{X}_i(s^-)\big)\big(\beta_{i} -
\hat{\beta}_{i}\big)\d \tilde{N}\big(\d t, \nu(\d\xi)\big),
\end{equation}
are martingales for all $i, j$. 
Then, the following inequality holds:

\begin{equation}
            \E\Big[\big|\bm{X}(t) -
              \tilde{\bm{X}}(t)\big|_2^{2}\Big]\leq
            \E\big[H(T)|\bm{X}(0)\big]
            \exp\Big(\big(2\overline{f}+1+(2\overline{\sigma} +1)m
            +(2\overline{\beta}+1)\nu(U)\big)T\Big),
            \label{pleq2}
\end{equation}
where $|\cdot|_2$ denotes the $2$-norm of a $d$-dimensional vector,
$\bm{X}(0)$ is the initial condition, and $H(t)$ is defined as
\begin{equation}
\begin{aligned}
  H(t) \coloneqq & \E\bigg[\sum_{i=1}^d\int_0^t\big(f_{i}(\bm{X}(s^-), s^-)
    - \hat{f}_{i}(\bm{X}(s^-), s^-)\big)^2\d s\bigg]\\
\: & \quad + \E\bigg[\sum_{i=1}^d\int_0^t\sum_{j=1}^m
  \big(\sigma_{i, j}(\bm{X}(s^-), s^-)
  - \hat{\sigma}_{i, j}(\bm{X}(s^-), s^-)\big)^2\d s\bigg]\\
\: & \qquad + \E\bigg[\sum_{i=1}^d\int_0^t\int_U\big(\beta_{i}(\bm{X}(s^-), \xi, s^-)
  - \hat{\beta}_{i}(\bm{X}(s^-), \xi, s^-)\big)^2\nu(\d\xi)\d s\bigg]
\end{aligned}
\label{h_define}
\end{equation}
\end{theorem}

The proof to Theorem~\ref{theorem1} is similar to the proof of the
stochastic Gronwall lemma (Theorem 2.2 in \cite{mehri2019stochastic})
and is given in~\ref{proof_theorem1}.  Theorem~\ref{theorem1} greatly
generalizes the results of Theorem~1 in \cite{xia2023a}, which was
developed for analyzing the $W_2$-distance between two one-dimensional
pure diffusion processes. Now, with Theorem~\ref{theorem1}, we can
analyze the $W$-distance between two multi-dimensional jump-diffusion
processes.

The following corollary establishes the upper bound of the
$W$-distance $W_p(\mu, \hat{\mu}), 1\leq p\leq 2$ between $\mu$ and
$\hat{\mu}$, the two probability distributions associated with
jump-diffusion processes Eqs.~\eqref{model_equation} and
~\eqref{approximate_equation}.

\begin{corollary}
\rm
\label{corollary1}
    (Upper error bound for the $W$-distance) The following bound holds
for $W_{p}(\mu, \hat{\mu})$, where $\mu$ and $\hat{\mu}$ are the two
probability distributions associated with jump-diffusion processes
Eqs.~\eqref{model_equation} and ~\eqref{approximate_equation}
\begin{equation}
W_{p}(\mu, \hat{\mu})\leq
\sqrt{\,T \E\big[H(T)\big|\bm{X}(0)\big]}\times
\exp\Big(\big(\overline{f}+\tfrac{1}{2}+(\overline{\sigma}+\tfrac{1}{2})
m + (\overline{\beta}+\tfrac{1}{2})\nu(U)\big)T\Big),
    \label{col2def}
\end{equation}
where $H(T)$ is defined in Eq.~\eqref{h_define}.
    
\end{corollary}

\begin{proof}
The proof of Corollary~\ref{corollary1} is a direct application of
Theorem~\ref{theorem1}.  We denote $\tilde{\mu}$ to be the
distribution of $\tilde{\bm{X}}$ defined in
Eq.~\eqref{tilde_equation}.  Since $\tilde{\bm{X}}$ has the same
distribution as $\hat{\bm{X}}$, we have, by the H\"older's inequality
\begin{equation}
  W_p^p(\mu, \hat{\mu})=W_p^p(\mu, \tilde{\mu})\leq \E\bigg[\!\int_0^T
    \!\big|\bm{X}(s) - \tilde{\bm{X}}(s)\big|_2^{2}\d
    s\,\Big|\,\bm{X}(0)\bigg]^{\frac{p}{2}}, \,\,\, 1\leq p\leq 2.
\end{equation}
Using Eq.~\eqref{pleq2} and the fact that $H(t)$ is non-decreasing
w.r.t. $t$, we have
\begin{equation}
 W_{p}(\mu, \hat{\mu})\leq
 \sqrt{\,T\E\big[H(T)\big|\bm{X}(0)\big]}\times
 \exp\Big(\big((\overline{f}+\tfrac{1}{2}) +
 (\overline{\sigma}+\tfrac{1}{2}) m +
 (\overline{\beta}+\tfrac{1}{2})\nu(E) \big)T\Big), \,\,\, 1\leq p\leq
 2.
\end{equation}
which proves Eq.~\eqref{col2def}.

\end{proof}

From Corollary~\ref{corollary1}, it is necessary to have a small
$W_p(\mu, \hat{\mu})$ in Eq.~\eqref{col2def} such that the errors in
the drift, diffusion, and jump functions $\bm{f}-\hat{\bm{f}}$,
$\bm{\sigma}-\hat{\bm{\sigma}}$, and $\bm{\beta}-\hat{\bm{\beta}}$ can
be small. Note that Corollary~\ref{corollary1} analyzes the classic
$W_p$-distance $W_p(\mu, \hat{\mu})$, which is different from the
smooth Wasserstein distance in \cite{breton2024wasserstein} (the
classical Wasserstein distance is an upper bound for the smooth
Wasserstein distance used in \cite{arras2019stein}).

$W_p(\mu, \hat{\mu}), 1\leq p\leq 2$ cannot be directly used as a loss
function to minimize as we cannot directly evaluate
$\|\bm{X}-\hat{\bm{X}}\|^{p}$ in Eq.~\eqref{pidef} since this term
requires evaluation of the integral $\int_0^T \sum_{i=1}^d
|X_i(t)-\hat{X}_i(t)|^2\d t$.
However, when $p=2$ ($W_2(\mu, \hat{\mu})$), we shall show that we can
efficiently estimate $W_2(\mu, \hat{\mu})$ by using finite-time-point
observations of the two jump-diffusion processes $\bm{X}(t)$ and
$\hat{\bm{X}}(t)$.

Let $0=t_0<t_1<...<t_N=T$ to be a mesh grid in the time interval
$[0, T]$, and we define the following projection operator $I_N$
\begin{equation}
\bm{X}_N(t) \coloneqq I_N \bm{X}(t) =
\begin{cases}
&\bm{X}(t_i), t\in[t_i, t_{i+1}),\quad  i<N-1,\\
&\bm{X}(t_i), t\in[t_i, t_{i+1}],\quad  i=N-1.
\end{cases}
\label{X_N_def}
\end{equation}
The projected $\bm{X}_N(t)$ in Eq.~\eqref{X_N_def} is piecewise
constant and is thus in the space $L^2([0, T])$. We denote the
distributions of $\boldsymbol{X}_N(t)$ and
$\hat{\boldsymbol{X}}_N(t)\coloneqq I_N\hat{\bm{X}}_N$ in
Eq. \ref{X_N_def} by $\mu_N$ and $\hat{\mu}_N$, respectively. We will
prove the following theorem for bounding the error $|W_2(\mu,
\hat{\mu})-W_2\left(\mu_N, \hat{\mu}_N\right)|$.

\begin{theorem}
\rm
\label{theorem3}
(Finite-time-point approximation for $W_2$ distance) The following
triangular inequality for $W_2(\mu, \hat{\mu})$ holds:
\begin{eqnarray}
\hspace{-0.9cm} W_2(\mu_N, \hat{\mu}_N) - W_2(\mu, \mu_N)
- W_2(\hat{\mu}, \hat{\mu}_N)
\leq W_2(\mu, \hat{\mu})
\leq W_2(\mu_N, \hat{\mu}_N) + W_2(\mu, \mu_N) + W_2(\hat{\mu}, \hat{\mu}_N).
\label{triangular}
\end{eqnarray}
In Eq.~\eqref{triangular}, $\mu_N, \hat{\mu}_N$ are the probability
distributions associated with $\bm{X}_N$ and $\hat{\bm{X}}_N$ defined
in Eq.~\eqref{X_N_def}, respectively.  Furthermore, we assume that
\begin{equation}
  \begin{aligned}
    F\coloneqq & \E\Big[\medint\int_0^T
      {\small \sum}_{i=1}^d f_i^2(\bm{X}(t^-),t^-)\d t\Big]<\infty,
    && \hat{F}\coloneqq \E\Big[\medint\int_0^T \sum_{i=1}^d
\hat{f}_i^2(\hat{\bm{X}}(t^-),t^-)\d t\Big]<\infty \\
\Sigma\coloneqq & \E\Big[\medint\int_0^T \sum_{\ell=1}^d
  \sum_{j=1}^m\sigma_{i, j}^2(\bm{X}(t^-),t^-)\d t\Big]<\infty, && \hat{\Sigma}
\coloneqq \E\Big[\medint\int_0^T
  \sum_{\ell=1}^d\sum_{j=1}^m
  \hat{\sigma}_{i, j}^2(\hat{\bm{X}}(t^-),t^-)\d t\Big]<\infty,\\
B\coloneqq & \E\Big[\medint\int_0^T \sum_{\ell=1}^d
  \medint\int_U \beta_{i}^2(\bm{X}(t^-),\xi, t^-)\nu(\d\xi)\d t\Big]<\infty, && \hat{B}
\coloneqq \E\Big[\medint\int_0^T
  \sum_{\ell=1}^d\medint\int_U \hat{\beta}_{i}^2(\hat{\bm{X}}(t^-),\xi, t^-)
  \nu(\d \xi)\d t\Big]<\infty,
\end{aligned}
\label{F_Sigma}
\end{equation}
where $\bm{X}(t)$ and $\hat{\bm{X}(t)}$ solve
Eqs.~\eqref{model_equation} and \eqref{approximate_equation},
respectively.
Then, we have the following bound
\begin{equation}
  \big|W_2(\mu_N, \hat{\mu}_N)-W_2(\mu,\hat{\mu})\big| \leq \sqrt{\Delta t}
  \left(\sqrt{F\Delta t+\Sigma+B}
+\sqrt{\hat{F}\Delta t+\hat{\Sigma} +\hat{B}}\right).
\label{dtbound}
\end{equation}
where $\Delta t\coloneqq \max_{i=0,...,N-1} |t_{i+1}-t_i|$.
\end{theorem}

The proof to Theorem~\ref{theorem3}, given in ~\ref{theorem3proof},
uses the It\^o isometry as well as the orthogonal assumption of the
compensated Poisson process in Assumption~\ref{assumptions}.
Theorem~\ref{theorem3} indicates that $W_2(\mu, \hat{\mu})$ can be
approximated by the finite-time-point projections $\bm{X}_N$ and
$\hat{\bm{X}}_N$.  Specifically, Theorem~\ref{theorem3} is a
generalization to Theorem~2 in \cite{xia2023a}, developed for the pure
diffusion. Note that Eq.~\eqref{dtbound} holds for $|W_2(\mu,
\hat{\mu}) - W_2(\mu_N, \hat{\mu}_N)|$ but Eq.~\eqref{dtbound} might
not hold for $|W_p(\mu, \hat{\mu}) - W_p(\mu_N, \hat{\mu}_N)|,\, 1\leq
p<2$ as we cannot directly apply the It\^o isometry to the compensated
Poisson process for $1\leq p<2$.

It has been shown in \cite{xia2023a} that for reconstructing pure
diffusion processes, minimizing a temporally decoupled squared $W_2$
distance can yield more accurate reconstructions of the drift and
diffusion functions than direct minimization of the squared $W_2$
distance $W_2^2(\mu, \hat{\mu})$ defined in
Eq.~\eqref{pidef}. Additionally, the squared $W_2$ distance $W_2(\mu,
\hat{\mu})$ is an upper bound of the temporally decoupled squared
$W_2$ distance that will be discussed in Section~\ref{section3}. Thus,
Corollary~\ref{corollary1} also applies to the temporally decoupled
squared $W_2$ distance.

\section{A temporally decoupled squared $W_2$ distance}
\label{section3}
In this section, we propose and analyze a temporally decoupled squared
$W_2$ distance, which could help effectively approximate the
jump-diffusion process Eq.~\eqref{model_equation} by the reconstructed
jump-diffusion process Eq.~\eqref{approximate_equation}. Specifically,
we show why this temporally decoupled squared $W_2$ distance can be
more effectively evaluated using empirical distributions, making it a
more appealing choice than the squared $W_2$ distance $W_2^2(\mu,
\hat{\mu})$ discussed in Section~\ref{section2}. The
\textbf{temporally decoupled squared} $\bm{W_2}$ \textbf{distance} is
defined as
\begin{equation}
  \tilde{W}_2^2(\mu, \hat{\mu})\coloneqq
  \int_0^T\! W_2^2\big(\mu(t), \hat{\mu}(t)\big)\d t,
    \label{decoupled_def}
\end{equation}
where $\mu(t), \hat{\mu}(t)$ are the distributions of $d$-dimensional
random variables $\bm{X}(t)$ and $\bm{X}(t)$ at time $t$,
respectively:
\begin{equation}
W_2^2(\mu(t), \hat{\mu}(t)) \coloneqq \!\inf_{\pi(\mu(t), \hat \mu(t))}
\!\E_{(\bm{X}(t), \hat{\bm{X}(t)})\sim \pi(\mu, \hat
  \mu)}\big[|{\bm{X}}(t) - \hat{{\bm{X}}}(t)|_2^{2}\big],
\end{equation}
where the joint distribution $\pi(\mu(t), \hat \mu(t))$ satisfies

\begin{equation}
    \pi(\mu(t), \hat \mu(t))(A,
    \mathbb{R}^d) = \mu(t)(A) ,\,\,\,\,\pi(\mu(t), \hat \mu(t))(\mathbb{R}^d, B
    ) = \hat{\mu}(t)(B),\,\,\,\, A, B\in\mathcal{B}(\mathbb{R}^d).
\end{equation}
    
%
The integration on the RHS of Eq.~\eqref{decoupled_def} is defined as
the limit

\begin{equation}
  \int_0^T\! W_2^2(\mu(t), \hat{\mu}(t))\d t
  =\lim\limits_{\max_i (t_{i+1}-t_i)\rightarrow 0}
  \sum_{i=0}^{N-1}W_2^2(\mu(t_i), \hat{\mu}(t_{i}))\Delta t_i,
    \label{tilde_def}
\end{equation}
where $0=t_0<t_1<...<t_N=T$ is a grid mesh on the time interval $[0,
  T]$ and $\Delta t_{i} \coloneqq t_{i+1}-t_{i}$ in the following. Here, we shall prove
that the temporally decoupled squared $W_2$ distance
$\tilde{W}_2^2(\mu, \hat{\mu})$ is well defined and can be a more
effective loss function when seeking to reconstruct multidimensional
jump-diffusion processes than the original squared $W_2^2(\mu_N,
\hat{\mu}_N)$.  Two features make this so: i) numerically evaluating
the temporally decoupled squared $W_2$ distance
Eq.~\eqref{decoupled_def} using finite-sample empirical distributions
can be more accurate than evaluating the original squared $W_2$
distance $W_2^2(\mu, \hat{\mu})$ when the number of training samples
becomes larger, and ii) the temporally decoupled squared $W_2$
distance Eq.~\eqref{decoupled_def} provides upper error bounds for
$\bm{f}-\hat{\bm{f}}, \bm{\sigma}-\hat{\bm{\sigma}}$, and $\bm{\beta}
- \hat{\bm{\beta}}$ when reconstructing jump-diffusion processes.

We denote $\bm{\mu}_i$ and $\hat{\bm{\mu}}_i$ to be the distributions
for $\bm{X}(t), t\in[t_i, t_{i+1})$ and $\hat{\bm{X}}(t), t\in[t_i,
    t_{i+1})$, respectively. We can prove the following theorem that
    shows the limit on the RHS of Eq.~\eqref{tilde_def} exists and
    thus the temporally decoupled squared $W_2$ in
    Eq.~\eqref{decoupled_def} distance is well defined.

\begin{theorem}
\rm
\label{theorem4}
(The temporally decoupled squared $W_2$ distance is well-defined) We
assume the conditions in Theorem~\ref{theorem3} hold. Furthermore, we
assume that for any $0<t<t'<T$, as $t'-t\rightarrow0$, the following
conditions are satisfied
\begin{equation}
 \begin{aligned}
\: & \E \bigg[\medint\int_{t}^{t'}\!\sum_{i=1}^d f_i^2(\bm{X}(t),t)\text{d}t\bigg]
  \rightarrow 0, && \E \bigg[\medint\int_{t}^{t'} \sum_{i=1}^d
    \hat{f}_i^2(\hat{\bm{X}}(s^-),s^-)\text{d}s\bigg]\rightarrow 0,\\
\: & \E\bigg[\medint\int_{t}^{t'} \!\sum_{\ell=1}^d
  \sum_{j=1}^m\sigma_{i, j}^2(\bm{X}(s^-),s^-)\d s\bigg]\rightarrow 0, &&
\E\bigg[\medint\int_{t}^{t'}\!\sum_{\ell=1}^d
\sum_{j=1}^m\hat{\sigma}_{i, j}^2(\hat{\bm{X}}(s^-),s^-)\d s^-\bigg]\rightarrow 0, \\
\: & \E\bigg[\medint\int_{t}^{t'}\!\sum_{\ell=1}^d\!
  \medint\int_U\beta_{\ell}^2(\bm{X}(s^-),\xi, s^-)\nu(\d\xi)\d s\bigg]\rightarrow 0, &&
\E\bigg[\medint\int_{t}^{t'}\!\sum_{\ell=1}^d\!
  \medint\int_U \hat{\beta}_{\ell}^2(\hat{\bm{X}}(s^-),\xi, s^-)\nu(\d\xi)\d s\bigg]
\rightarrow 0.
\end{aligned}
\label{condition_dt}
\end{equation}
Additionally, we assume that there is a uniform upper bound

\begin{equation}
M\coloneqq \max_{t\in[0, T]} W_2\big(\mu(t), \hat{\mu}(t)\big)\leq \infty.
\label{M_condition}
\end{equation}
Suppose $\Delta t \coloneqq \max_{0\leq i\leq N-1}\Delta t_i$, then
%
\begin{equation}
\lim\limits_{\Delta
  t\rightarrow 0}\Big(\sum_{i=0}^{N-1}W_2^2\big(\mu(t_i),
\hat{\mu}(t_i)\big)\Delta t_{i} -
\sum_{i=0}^{N-1}W_2(\boldsymbol{\mu}_i, \hat{\boldsymbol{\mu}}_i)
\Big)= 0.
\end{equation}
Furthermore, the limit
\begin{equation}
  \lim\limits_{\Delta t\rightarrow 0}\sum_{i=0}^{N-1}W_2^2\big(\mu(t_i),
  \hat{\mu}(t_i)\big)\Delta t_{i} =
  \lim\limits_{N\rightarrow \infty}\sum_{i=0}^{N-1}W_2^2\Big(\mu(t_i),
  \hat{\mu}(t_i)\Big)\Delta t_{i}
  \label{lim_condition}
\end{equation}
is simply $\tilde{W}_2^2(\mu, \hat{\mu})$ defined in
Eq.~\eqref{decoupled_def}. Denoting $\pi_i$ to be the coupling
probability distribution of $(\bm{X}(t_i), \hat{\bm{X}}(t_i))$, whose
marginal distributions coincide with $\mu(t_i)$ and $\hat{\mu}(t_i)$,
we have the following bound:

\begin{equation}
  \Big|\sum_{i=0}^{N-1}\inf_{\pi_i}\E_{\pi_i}\big[|\bm{X}(t_i)
  - \hat{\bm{X}}(t_i)|_2^2\big]\Delta t_{i}
- \tilde{W}_2^2(\mu, \hat{\mu})\Big| \leq 2M T\max_i\Big(\sqrt{F_i\Delta t+\Sigma_i+B_i}
  +\!\sqrt{\hat{F}_i\Delta t + \hat{\Sigma}_i + \hat{B}_i}\Big),
     \label{convergence_result}
\end{equation}
where
\begin{equation}
  \begin{aligned}
    F_i\coloneqq & \E\bigg[\medint\int_{t_i}^{t_{i+1}}
      \sum_{i=1}^d f_i^2(\bm{X}(t^-),t^-)\d t\bigg], &&
    \hat{F}_i\coloneqq\E\bigg[\medint\int_{t_i}^{t_{i+1}} \sum_{i=1}^d
\hat{f}_i^2(\hat{\bm{X}}(t^-),t^-)\d t\bigg], \\
\Sigma_i\coloneqq & \E\bigg[\medint\int_{t_i}^{t_{i+1}} \sum_{\ell=1}^d
  \sum_{j=1}^m\sigma_{\ell, j}^2(\bm{X}(t^-),t^-)\d t\bigg], &&
\hat{\Sigma}_i\coloneqq\E\bigg[\medint\int_{t_i}^{t_{i+1}}
\sum_{\ell=1}^d\sum_{j=1}^m\hat{\sigma}_{\ell, j}^2(\hat{\bm{X}}(t^-),t^-)\d t\bigg],\\
B_i\coloneqq & \E\bigg[\medint\int_{t_i}^{t_{i+1}} \sum_{\ell=1}^d
\medint\int_U \beta_{\ell}^2(\bm{X}(t^-),\xi, t^-)\nu(\d\xi)\d t\bigg], && 
\hat{B}_i\coloneqq\E\bigg[\medint\int_{t_i}^{t_{i+1}}
  \sum_{\ell=1}^d\medint\int_U \hat{\beta}_{\ell}^2(\hat{\bm{X}}(t^-),\xi, t^-)
  \nu(\d\xi)\d t\bigg].
  \end{aligned}
  \end{equation}

\end{theorem}

Theorem~\ref{theorem4} generalizes Theorem~3 in \cite{xia2023a} from
pure diffusion processes to jump-diffusion processes. The proof to
Theorem~\ref{theorem4} is in ~\ref{theorem4proof}. Specifically, from
Eq.~\eqref{convergence_result}, if $\max_i (F_i\Delta t + \Sigma_i+B_i
+ \hat{F}_i\Delta t + \hat{\Sigma}_i + \hat{B}_i)$ is of order $\Delta
t$, then the convergence rate of
$\sum_{i=1}^{N-1}\inf_{\pi_i}\E_{\pi_i}[|\bm{X}(t_i) -
  \hat{\bm{X}}(t_i)|_2^2]\Delta t_{i}$ to $\tilde{W}_2^2(\mu,
\hat{\mu})$ is $O(\sqrt{\Delta t})$. Specifically, we have

\begin{equation}
  \big|W_2^2(\mu, \hat{\mu}) - W_2^2(\mu_N, \hat{\mu}_N)\big|
  = \big|W_2(\mu, \hat{\mu})
  - W_2(\mu_N, \hat{\mu}_N)\big|\cdot \big|W_2(\mu, \hat{\mu})
  + W_2(\mu_N, \hat{\mu}_N)\big|.
\end{equation}
From Eq.~\eqref{dtbound}, the error bound of $|W_2(\mu, \hat{\mu}) -
W_2(\mu_N, \hat{\mu}_N)|$ is $O\big(\sqrt{\Delta t}\big)$. Therefore,
the upper error bounds of using the finite-time distributions
$\mu_N,\hat{\mu}_N$ to approximate both $W_2^2(\mu, \hat{\mu})$ or
$\tilde{W}_2^2(\mu, \hat{\mu})$ are both of order $O(\sqrt{\Delta
  t})$.

Next, we shall show that using the finite-sample empirical
distribution to estimate

\begin{equation}
\sum_{i=0}^{N-1}\inf_{\pi_i}\E_{\pi_i}\left[|\bm{X}(t_i)
    - \hat{\bm{X}}(t_i)|_2^2\right]\Delta t_{i},
\end{equation}
where $\pi_i$ is the coupling distribution of $(\bm{X}(t_i),
\hat{\bm{X}}(t_i))$ such that its marginal distributions are
$\mu(t_i)$ and $\hat{\mu}(t_i)$, is more accurate than using the
finite-sample empirical distribution to estimate $W_2^2(\mu_N,
\hat{\mu}_N)$ (where $\mu_N$ and $\hat{\mu}_N$ are the distributions
of $\bm{X}_N(t)$ and $\hat{\bm{X}}_N(t)$ defined in
Eq.~\eqref{X_N_def}).

\begin{theorem}
\rm
\label{theorem5}
(Finite sample empirical distribution error bound) We assume that

\begin{equation}
  \E\big[|\bm{X}(t)|_6^6\big]\leq \infty, \,\,\,
  \E\big[|\hat{\bm{X}}(t)|_6^6\big]\leq \infty, \,\,\, \forall t\in[0, T],
\end{equation}
where $|\cdot|_6$ is the $l^6$ norm of a vector in $\mathbb{R}^d$.  We
denote $\mu_N^{\e}, \hat{\mu}_N^{\e}$ to be empirical distributions of
$\bm{X}_N$ and $\hat{\bm{X}}_N$, respectively; we denote
$\mu_N^{\e}(t_i), \hat{\mu}_N^{\e}(t_i)$ to be the empirical
distributions of $\bm{X}(t_i)$ and $\hat{\bm{X}}(t_i),
i=0,1,...,N-1$. Suppose $M_s$ is the number of observed trajectories
$\bm{X}_N(t_i)$ and the number of reconstructed trajectories
$\hat{\bm{X}}_N(t_i)$. We find the following error bound for
estimating $W_2^2(\mu_N, \hat{\mu}_N)$ using the empirical
distributions:
\begin{equation}
\begin{aligned}
\: & \E\big[|W_2^2(\mu_N^{\e}, \hat{\mu}_N^{\e})
    - W_2^2(\mu_N, \hat{\mu}_N)|\big] \leq E_1(M_s), \quad\mbox{where} \\      
\: & E_1(M) \coloneqq 2\sqrt{C_0} W_2(\mu_N, \hat{\mu}_N) h(M_s, Nd)
\sum_{i=0}^{N-1}\Big(\E\Big[|\bm{X}(t_i)|_6^6\Big]^{\frac{1}{6}}+
\E\Big[|\hat{\bm{X}}(t_i)|^6_6\Big]^{\frac{1}{6}}\Big)\sqrt{\Delta t_{i}} \\
\: & \hspace{2.6cm} + 2C_0h^{2}(M_s,Nd)\sum_{i=0}^{N-1}\Big(
\E\Big[\big|\bm{X}(t_i)\big|_6^6\Big]^{\frac{1}{3}} +
\E\Big[\big|\hat{\bm{X}}(t_i)\big|_6^6\Big]^{\frac{1}{3}}\Big)\Delta t_{i},
\end{aligned}
\label{coupled_error_bound}
\end{equation}
where $C_0$ is a constant and

\begin{equation}
h(M_s, n)\coloneqq\left\{
\begin{aligned}
&M_s^{-\frac{1}{4}}\log(1+M_s)^{\frac{1}{2}},\,\,\, n\leq4,\\
&M_s^{-\frac{1}{n}},\quad  n> 4
\end{aligned}
\right.
\label{t_def}
\end{equation}
We also have the empirical error bound for estimating
$\sum_{i=1}^{N-1}W_2^2(\mu_N(t_i), \mu_N(t_i))\Delta t_i$
using the empirical distributions
$\sum_{i=1}^{N-1}W_2^2(\mu_N^{\e}(t_i),
\mu_N^{\e}(t_i))\Delta t_i$:
\begin{equation}
  \begin{aligned}
    \: & \E\Big[\big|\sum_{i=0}^{N-1}W_2^2(\mu_N^{\e}(t_i),
      \mu_N^{\e}(t_i))\Delta t_i
      - W_2^2(\mu_N(t_i), \mu_N(t_i))\Delta t_i\big|\Big] \hspace{4cm} \\
    \: &\qquad\leq
    \E\Big[\sum_{i=0}^{N-1}\big|W_2^2(\mu_N^{\e}(t_i), \mu_N^{\e}(t_i))
      - W_2^2(\mu_N(t_i), \mu_N(t_i))\big|\Delta t_i\Big]
    \leq E_2(M_s),\quad \mbox{where}\\
    \: & E_2(M_s)  \coloneqq 2\sqrt{C_1}h(M_s, d)
    \sum_{i=0}^{N-1}\Big(\E\Big[\big|\bm{X}(t_i)\big|_6^6\Big]^{\frac{1}{6}}\!
 +\E\Big[\big|\hat{\bm{X}}(t_i)\big|_6^6\Big]^{\frac{1}{6}}\Big)
\Delta t_iW_2\big(\mu_N(t_i), \hat{\mu}_N(t_i)\big) \\
 \: & \qquad \hspace{3.5cm} + 2C_1 h^{2}(M_s, d)\sum_{i=0}^{N-1}
\Big(\E\Big[\big|\bm{X}(t_i)\big|_6^6\Big]^{\frac{1}{3}}\!+
 \E\Big[\big|\hat{\bm{X}}(t_i)\big|_6^6\Big]^{\frac{1}{3}}\Big)\Delta t_i,
\end{aligned}
\label{decoupled_error_bound}
\end{equation}
where $C_1$ is a constant different from $C_0$.  Furthermore, there
exists a constant $C$ such that
\begin{equation}
   E_1(M_s)\geq C E_2(M_s)\cdot \frac{h(M_s, Nd)}{h(M_s, d)} N^{-\frac{2}{3}}.
\end{equation}
\end{theorem}

The proof of Theorem~\ref{theorem5} is given in ~\ref{proof_theorem5}
and utilizes the upper bound of the $W$-distance between the
ground truth distribution and the empirical distribution in
\cite{fournier2015rate}. Specifically, if $N\geq 5$, then $h(M_s,
Nd)=2M_s^{-\frac{1}{Nd}}$, and

\begin{equation}
  \frac{h(M_s, Nd)}{h(M_s, d)}\geq
  \min\Big\{M_s^{\frac{1}{4}-\frac{1}{Nd}}\log(1+M_s), M_s^{\frac{N-1}{Nd}}\Big\}.
\end{equation}
Therefore, Theorem~\ref{theorem5} indicates that as the number of
observed trajectories of the jump-diffusion process $M_s$ increases,
the upper bound of $\E\Big[\big|\sum_{i=0}^{N-1}W_2^2(\mu_N^{\e}(t_i),
  \mu_N^{\e}(t_i))\Delta t_i - \sum_{i=0}^{N-1}W_2^2\big(\mu_N(t_i),
  \mu_N(t_i)\big)\Delta t_i\big|\Big]$ converges faster to 0 than the
upper bound of $\E\Big[\big|W_2^2(\mu_N^{\e}, \hat{\mu}_N^{\e}) -
  W_2^2(\mu_N, \hat{\mu}_N)\big|\Big]$ does when $N\geq 5$.  Thus,
\begin{equation}
    \sum_{i=0}^{N-1}W_2^2\big(\mu_N(t_i), \mu_N(t_i)\big)\Delta t_i
    \label{temporal_calculation}
\end{equation}
can be more accurately evaluated by the finite-sample empirical
distributions than the squared $W_2$ distance $W_2^2(\mu_N,
\hat{\mu}_N)$ when $M_s$ is large.

For any coupled distribution $\pi(\bm{X}_N, \hat{\bm{X}}_N)$ such that
its marginal distributions are $\mu_N$ and $\hat{\mu_N}$, its marginal
distributions w.r.t $\bm{X}(t_i)$ and $\hat{\bm{X}}(t_i)$ are
$\mu(t_i)$ and $\hat{\mu}(t_i)$, respectively. Thus,

\begin{equation}
\sum_{i=0}^{N-1}\inf_{\pi_i} \E_{\pi_i}\Big[\big|\bm{X}(t_i) -
  \hat{\bm{X}}(t_i)\big|_2^2\Big]\Delta t_i
\leq\!\! \inf_{\pi(\bm{X}_N, \hat{\bm{X}}_N)}\!\sum_{i=0}^{N-1}
\E_{\pi(\bm{X}_N, \hat{\bm{X}}_N)}\Big[\big|\bm{X}(t_i) -
  \hat{\bm{X}}(t_i)\big|_2^2\Big]\Delta t_i = \!W_2^2(\mu_N, \hat{\mu}_N).
\label{lower_bound_w}
\end{equation}
Letting $N\rightarrow\infty$ in Eq.~\eqref{lower_bound_w}, from
Theorems~\ref{theorem3} and~\ref{theorem4}, we conclude that

\begin{equation}
  \tilde{W}_2^2(\mu, \hat{\mu})\leq 
  W_2^2(\mu, \hat{\mu}).
\end{equation}
Thus, Corollary~\ref{corollary1} also provides an upper error bound
for the temporally decoupled $\tilde{W}_2^2(\mu, \hat{\mu})$.  Next,
we show that there is a lower bound for $\tilde{W}_2^2(\mu,
\hat{\mu})$ and this lower bound depends on drift, diffusion, and jump
functions of the ground truth jump-diffusion process
Eq.~\eqref{model_equation} and the approximate jump-diffusion process
Eq.~\eqref{approximate_equation}.

\begin{theorem}
\label{theorem6}
\rm (Lower error bound for the temporally decoupled squared
$W_2$-distance) We have the following lower bound:
\begin{equation}
 \tilde{W}_2^2(\mu, \hat{\mu})
  \geq \!\medint\int_0^T\!\sum_{i=1}^d\bigg(
  \E\Big[\!\medint\int_0^t \!\big[f_i(\bm{X}(s^-), s^-)
    -\hat{f}_i(\hat{\bm{X}}(s^-), s^-)\big]\text{d}s\Big]\bigg)^2\text{d}t
  +\!\medint\int_0^T\!\!\text{Tr}\Big(\bm{S}_t+\hat{\bm{S}}_t
  - 2(\bm{S}_t\hat{\bm{S}}_t)^{\frac{1}{2}}\Big)\text{d}t,
        \label{lowererrorbound}
\end{equation}
where $\bm{S}_t, \hat{\bm{S}}_t$ are two matrices in
$\mathbb{R}^{d\times d}$ with their elements defined by

\begin{equation}
  \begin{aligned}
    (\bm{S}_t)_{i,j} \coloneqq & \E \Big[\sum_{\ell=1}^m\medint\int_0^t
      \sigma_{i, \ell}(\bm{X}(s^-), s^-)\cdot
      \sigma_{j, \ell}(\bm{X}(s^-), s^-)\text{d}s\Big] \\
    \: & \qquad\quad + \E \Big[\medint\int_0^t\!\medint\int_U
      \beta_{i}(\bm{X}(s^-),\xi, s^-)\cdot
      \beta_{j}(\bm{X}(s^-),\xi, s^-)
      \nu(\text{d}\xi)\text{d}s\Big], \\
    (\hat{\bm{S}}_t)_{i,j} \coloneqq & \E\Big[\sum_{\ell=1}^m
      \medint\int_0^t \hat{\sigma}_{i, \ell}(\hat{\bm{X}}(s^-), s^-)\cdot
      \hat{\sigma}_{j, \ell}(\hat{\bm{X}}(s^-), s^-)
      \text{d}s\Big] \\
    \: & \qquad\quad + \E \Big[\medint\int_0^t\!
      \medint\int_U \beta_{i}(\bm{X}(s^-), \xi, s^-)\cdot
      \beta_{j}(\bm{X}(s^-), \xi, s^-)\nu(\text{d}\xi)
      \text{d}s\Big].
  \end{aligned}
  \end{equation}
        The square roots $(\bm{S}_t)^{\frac{1}{2}}$ and $(\hat{\bm{S}}_t)^{\frac{1}{2}}$ are the positive roots.
\end{theorem}

\begin{proof}
First, we denote
  
\begin{equation}
\bm{X}_0(t)\coloneqq \bm{X}(t) -
\E[\bm{X}(t)],\,\,\,\hat{\bm{X}}_0(t)\coloneqq \hat{\bm{X}}(t) -
\E[\hat{\bm{X}}(t)]
\end{equation}
and let $\mu_0(t)$ and $\hat{\mu}_0(t)$ to be the probability
distributions of $\bm{X}_0(t)$ and $\hat{\bm{X}}_0(t)$, respectively.
From Theorem 1 in \cite{dowson1982frechet}, we have
\begin{equation}
  W_2^2(\mu_0(t), \hat{\mu}_0(t)) \geq
  \text{Tr}(\bm{S}_t+\hat{\bm{S}}_t - 2(\bm{S}_t\hat{\bm{S}}_t)^{\frac{1}{2}}).
\end{equation}
Because

\begin{equation}
  W_2^2\big(\mu(t), \hat{\mu}(t)\big)=W_2^2(\mu_0(t), \hat{\mu}_0(t)) +
  \sum_{i=1}^d\left(\E\bigg[\int_0^t \Big(f_i\big(\bm{X}(s^-), s^-\big)
    -\hat{f}_i\big(\hat{\bm{X}}(s^-), s^-\big)\Big)\text{d}s\bigg]\right)^2,
\end{equation}
Eq.~\eqref{lowererrorbound} holds, proving Theorem~\ref{theorem6}.

\end{proof}

Theorem~\ref{theorem6} gives a lower bound for the temporally
decoupled $\tilde{W}(\mu, \hat{\mu})$. Specifically, if $d=1$,
Eq.~\eqref{lowererrorbound} can be further simplified to

\begin{equation}
  \begin{aligned}
    \tilde{W}_2^2(\mu, \hat{\mu})\geq & \int_0^T\bigg(\E\Big[\medint
      \int_0^t f_1(X(s^-),s^-)\d s\Big] -
    \E\Big[\medint\int_0^t\hat{f}_1(\hat{X}(s^-), s^-)\text{d}s\Big]\bigg)^2
    \text{d}s \\
    \: & \,\,  + \int_0^T \bigg(\E\Big[\medint\int_0^t
      \sigma^2(X(s^-),s^-)\text{d}s+\medint\int_U
      \beta^2\big(X(s^-), \xi,s^-\big)\nu(\d\xi)\text{d}s\bigg]^{\frac{1}{2}} \\
    \: & \qquad\qquad-\E\Big[\medint\int_0^t
      \hat{\sigma}^2\big(\hat{X}(s^-),s^-\big)\text{d}s
      + \medint\int_U \hat{\beta}^2\big(\hat{X}(s^-), \xi,s^-\big)\nu(\d\xi)
      \text{d}s\Big]^{\frac{1}{2}}\bigg)^2\text{d}t.
    \end{aligned}
    \label{one_d_lower}
\end{equation}
Thus, if the jump-diffusion process to be reconstructed is
one-dimensional (Eq.~\eqref{model_equation}), then we conclude that, as
$\tilde{W}_2^2(\mu, \hat{\mu})\rightarrow 0$,

\begin{equation}
  \begin{aligned}
  & \E\Big[\medint\int_0^t f_1(X_1(s^-),s^-)\d s\Big]
    -\E\Big[\medint\int_0^t\hat{f}_1(\hat{X}(s^-), s^-)\text{d}s\Big] \rightarrow 0,
    \,\,\,\text{a.s.} \quad \mbox{and}\\
    \: & \E\Big[\medint\int_0^t \sigma^2(X_1(s^-),s^-)\text{d}s
      + \medint\int_U \beta^2(X_1(s^-), \xi,s^-)\nu(\d\xi)\big)
      \text{d}s\Big]^{\frac{1}{2}} \\
    & \qquad\quad - \E\Big[\medint\int_0^t
      \hat{\sigma}^2\big(\hat{X}_1(s^-),s^-\big)\text{d}s
  + \medint\int_U \hat{\beta}^2\big(\hat{X}_1(s^-), \xi,s^-\big)\nu(\d\xi)\text{d}s
  \Big]^{\frac{1}{2}}\rightarrow 0, \,\,\,\text{a.s.}
\end{aligned}
    \label{one_implication}
  \end{equation}
However, Eq.~\eqref{one_implication} does not imply that either
  
\begin{equation}
  \E\Big[\medint\int_0^t \hat{\sigma}^2(\hat{X}_1(s^-),s^-)\d s\Big]
  - \E\Big[\medint\int_0^t \hat{\sigma}^2(\hat{X}_1(s^-),s^-)\d s\Big]\rightarrow 0
\end{equation}
or

\begin{equation}
  \E\Big[\medint\int_0^t\!
    \medint\int_U \hat{\beta}^2(X_1(s^-), \xi,s^-)\nu(\d\xi)\text{d}s\Big] -
  \E\Big[\medint\int_0^t\!
    \medint\int_U \hat{\beta}^2(X_1(s^-), \xi,s^-)\nu(\d\xi)\text{d}s\Big] \rightarrow 0.
\end{equation}

A lower bound for the temporally decoupled squared $W_2$ distance
between two jump-diffusion processes that depends on the expectation
of the summation of the error in the jump and the error in the
diffusion functions is worth further investigation. Such an intricate
analysis is beyond the scope of this paper but could imply that
minimizing the $W_2$ distance is necessary for a good reconstruction
of both the diffusion function and the jump function.  Moreover, when
$d>1$, it is not easy to make further simplifications to
Eq.~\eqref{lowererrorbound}. Analysis of the properties of the matrix
$\bm{S}_t+\hat{\bm{S}}_t - 2(\bm{S}_t\hat{\bm{S}}_t)^{\frac{1}{2}}$ in
Eq.~\eqref{lowererrorbound} can be quite difficult.  Nonetheless, we
shall show in our numerical examples that our temporally decoupled
squared $W_2$-distance $\tilde{W}(\mu, \hat{\mu})$ method can
accurately reconstruct both the diffusion and the jump functions in
several examples of one-dimensional and multidimensional
jump-diffusion processes especially when the drift function can be
provided as prior information.

\section{Numerical experiments}
\label{section4}
In this section, we implement our methods through numerical examples
and investigate the effectiveness of the temporally decoupled squared
$W_2$-distance method for reconstructing the jump-diffusion process
Eq.~\eqref{model_equation}. We also compare our results with those
derived from using other commonly used losses in uncertainty
quantification and methods for jump-diffusion process
reconstruction. Additionally, we explore how prior knowledge on the
ground truth jump-diffusion process Eq.~\eqref{model_equation} helps
in its reconstruction. All experiments are carried out using Python
3.11 on a desktop with a 32-core Intel® i9-13900KF CPU (when comparing
runtimes, we train each model on just one core).


In all experiments, we use three feed-forward neural networks to
parameterize the drift, diffusion, and jump functions in the
approximate jump-diffusion process Eq.~\eqref{approximate_equation},
\textit{i.e.},
\begin{equation}
    \hat{f}\coloneqq \hat{f}(X, t;\Theta_1),\,\,\, \hat{\sigma}\coloneqq
  \hat{\sigma}(X, t;\Theta_2), \,\,\,\hat{\beta}\coloneqq
  \hat{\sigma}(X, \xi, t;\Theta_3).
\end{equation}
$\Theta_1, \Theta_2, \Theta_3$ are the parameter sets in the three
parameterized neural networks, respectively.  We modified the
\texttt{torchsde} Python package in \cite{li2020scalable} to implement
the Euler-Maruyama scheme for generating trajectories of the two
jump-diffusion processes Eqs.~\eqref{model_equation} and
~\eqref{approximate_equation}. Details of the training settings and
hyperparameters for all examples are given in
~\ref{training_details}. In Examples~\ref{example1} and
~\ref{example2}, the reconstruction errors are the relative $L^2$
errors:

\begin{align}
  \text{drift error}\coloneqq & \frac{\sum_{i=0}^N\sum_{j=1}^{M_s}\big|f(x_j(t_i), t_i)
    - \hat{f}(x_j(t_i), t_i)\big|}{\sum_{i=0}^N\sum_{j=1}^{M_s}|f(x_j(t_i), t_i)|},
  \label{drift_error}\\
 \text{diffusion error}\coloneqq & 
  \frac{\sum_{i=0}^N\sum_{j=1}^{M_s}\big||\sigma(x_j(t_i), t_i)|
   - |\hat{\sigma}(x_j(t_i), t_i)|\big|}{\sum_{i=0}^N\sum_{j=1}^{M_s}\big|\sigma(x_j(t_i), t_i)\big|}
    \label{diffusion_error}\\
\text{jump error}\coloneqq & \frac{\sum_{i=0}^N\sum_{j=1}^{M_s}\int_U \big|\beta(x_j(t_i), \xi, t_i)
  - \hat{\beta}(x_j(t_i), \xi, t_i)\big|\text{d}\nu(\xi)}{\sum_{i=0}^N\sum_{j=1}^{M_s}
  \int_D\big|\beta(x_j(t_i), t_i)\big|\text{d}\nu(\xi)},
\label{jump_error}
\end{align}
where $N$ is the number of time steps and $M_s$ is the number of
training trajectories.

\begin{example}
\rm
\label{example1}
For our first example, we reconstruct the following 1D jump-diffusion
process for describing the non-defaultable zero-coupon bond pricing
\cite{jang2007jump}:
\begin{equation}
        \d X_t = (b+aX_t)\d t +\sigma_0\sqrt{|X_t|}\d B_t 
        + \d C_t, \,\,\, a,b,\sigma_0\in\mathbb{R}, t\in[0, T]
        \label{example1_model}
    \end{equation}
where $C_t=\sum_{i=1}^{N_t}Y_i$, $Y_i$ are independently identically
distributed, and $N_t$ obeys the Poisson distribution with intensity
$t$. We take $Y_i\equiv y_0$ so that Eq.~\eqref{example1_model} can be
rewritten as
\begin{equation}
 \d X_t = (b+y_0+aX_t)\d t +\sigma_0\sqrt{|X_t|}\d B_t + y_0\d
 \tilde{N}_t,\,\,\, t\in[0, T],
\label{example1_numerical}
\end{equation}
with $\tilde{N}_t$ a 1D compensated Poisson process with intensity
$t$. We define ground truth by $b=4, a=-1, \sigma_0=0.4, y_0=1$ in
Eq.~\eqref{example1_numerical} and take $T=20.2$ and initial condition
$X_0=2$.  We reconstruct Eq.~\eqref{example1_numerical} by minimizing
the temporally decoupled $W$-distance Eq.~\eqref{temporal_calculation}.

\begin{figure}
  \centering
  \hspace{1.5cm}\includegraphics[width=0.88\linewidth]{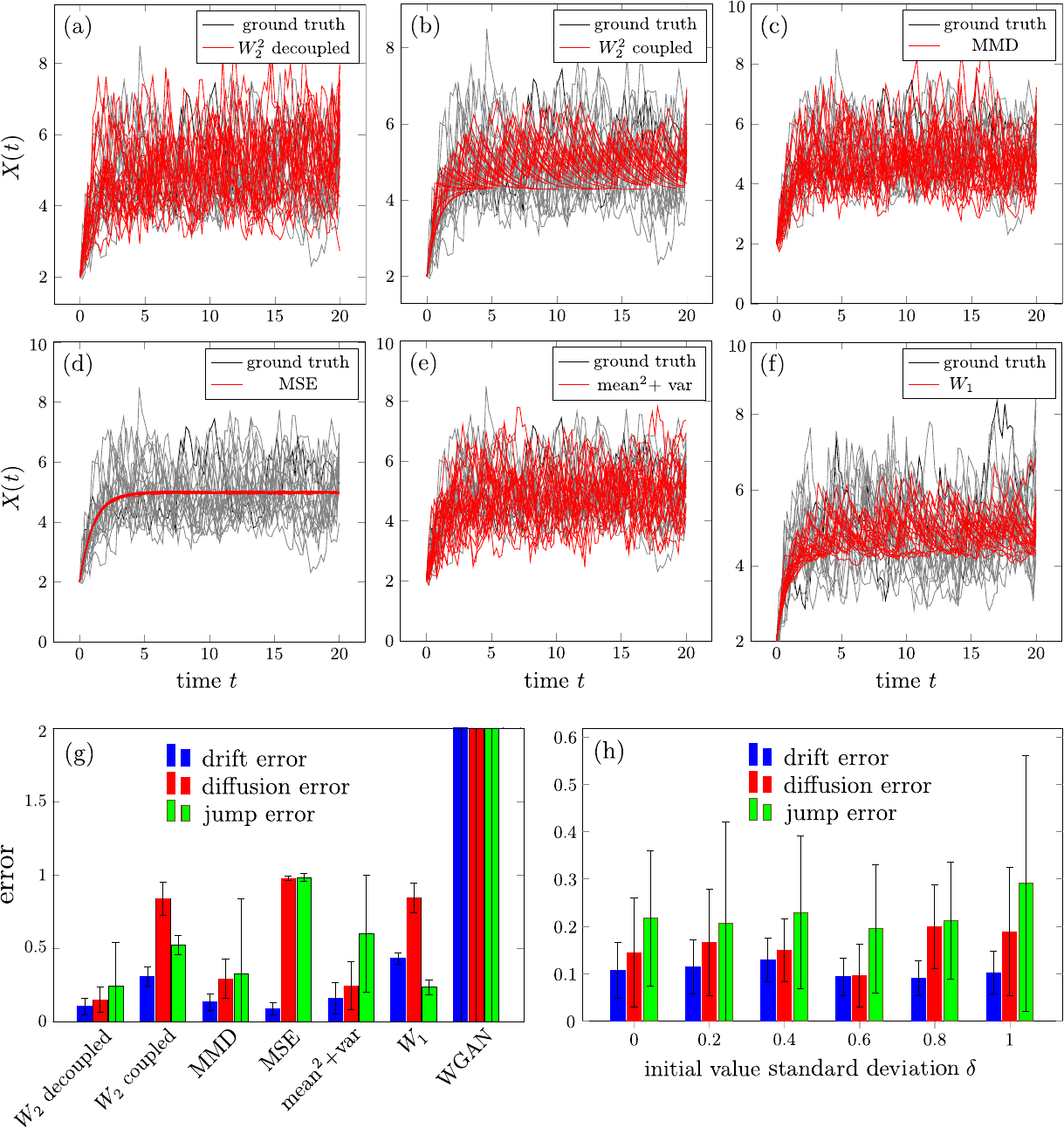}
    \caption{Reconstruction of trajectories and model functions. We
      define ground truth as $b=4, a=-1, \sigma_0=0.4, y_0=1$ in
      Eq.~\eqref{example1_numerical}, with $T=20.2$ and initial
      condition $X_0=2$.  (a-f) ground truth (black) and reconstructed
      trajectories (red) generated from the learned jump-diffusion
      process by minimizing different loss functions or using
      different methods.  (g) The reconstruction errors of the drift,
      diffusion, and jump functions defined in
      Eqs.~\eqref{drift_error}, \eqref{diffusion_error}, and
      \eqref{jump_error}. We compare errors from minimizing our
      temporally decoupled squared $W_2$-distance versus those from
      minimizing the MSE, MMD, mean$^2$+var, the $W_1$-distance
      $W_1(\mu, \hat{\mu})$, the squared $W_2$-distance $W_2^2(\mu_N,
      \hat{\mu}_N)$, and the error of results obtained using the WGAN
      method. The mean and standard deviation of the error for
      different methods are obtained by repeating the experiment 10
      times. (h) The reconstruction errors in the drift, diffusion,
      and jump functions defined in Eqs.~\eqref{drift_error},
      \eqref{diffusion_error}, and \eqref{jump_error} w.r.t. the
      standard deviation $\delta$ of the initial condition
      (Eq.~\eqref{IC_noise}).}
    \label{fig:example1}
\end{figure}

We compare our temporally decoupled squared $W_2$ distance loss
function with the WGAN method and other loss functions (MSE, MMD,
mean$^2$+var, $W_1$ distance, and the squared $W_2$ distance
$W_2^2(\mu_N, \hat{\mu}_N)$. The definitions of the other loss
functions are given in ~\ref{def_loss}). As shown in
Fig.~\ref{fig:example1}(a-f), the trajectories we obtained by
minimizing our temporally decoupled squared $W_2$-distance accurately
match the ground truth trajectories generated by
Eq.~\eqref{example1_numerical}. When using $W_1(\mu, \hat{\mu})$,
$W_2^2(\mu, \hat{\mu})$, and MSE loss functions, the reconstructed
trajectories deviate qualitatively from those of the ground truth.
The solutions of the reconstructed jump-diffusion process generated by
the WGAN method are also qualitatively incorrect and are thus not
shown here. From Fig.~\ref{fig:example1}(g), minimizing our temporally
decoupled squared $W_2$-distance gives the smallest reconstruction
errors $f-\hat{f}, \sigma-\hat{\sigma}$, and $\beta-\hat{\beta}$. The
average errors in the reconstructed drift, diffusion, and jump are
kept below 0.25. Thus, minimizing our temporally decoupled squared
$W_2$ distance is found to be more accurate in reconstructing the
jump-diffusion process Eq.~\eqref{example1_numerical} than other
benchmark methods.  We also list the average runtime per training
iteration as well as the memory usage of different methods in
Table~\ref{tab:example1_time}. The runtime of the WGAN method is
significantly longer than that of other methods. Furthermore, the
computational cost of using our temporally decoupled squared $W_2$ is
similar to the cost of using other loss functions while our temporally
decoupled squared $W_2$ method can accurately reconstruct
Eq.~\eqref{example1_numerical}.
\begin{table}[h!]
\centering
  \caption{The runtime and memory usage of different methods
    (loss functions) to reconstruct the jump-diffusion process
    Eq.~\eqref{example1_numerical}.}
  \vspace{2mm}
  \small
  \begin{tabular}{lccccccc}
\toprule  
Method (loss)  & \multirow{2}{*}{} temporally & & & & & & \\
& decoupled $W_2^2$ & $W_2^2$ & MMD  & MSE & mean$^2$+var & $W_1$ & WGAN\\ 
\midrule
Average time/iteration (s) & 44.5 & 48.2 & 59.4 & 29.8 & 48.7 & 39.2 & 346.3\\ 
Average memory use (Gb) & 2.61 &  3.52 & 2.59 & 5.00 & 2.61 & 2.60 & 3.34\\ 
\bottomrule
\end{tabular}
\label{tab:example1_time}
\end{table}

We test the numerical performance of our temporally decoupled squared
$W_2$ method when reconstructing Eq.~\eqref{example1_numerical} under
different initial conditions.  Instead of using the same initial
condition for all solutions, we sample the initial value from

\begin{equation}
    X_0 \sim \mathcal{N}(2, \delta^2),
    \label{IC_noise}
\end{equation}
where $\mathcal{N}(2, \delta^2)$ is the 1D normal distribution of mean
2 and variance $\delta^2$.  Using the same hyperparameters in the
neural networks and for training (in Table~\ref{training_details}) as
in Example~\ref{example1}, we varied the standard deviation $\delta=0,
0.2, 0.4, 0.6, 0.8, 1$ and implemented the temporally decoupled
squared $W_2$ distance as a loss function. The results shown in
Fig.~\ref{fig:example1}(h) indicate that the reconstruction of
Eq.~\eqref{example1_numerical} using the squared $W_2$ loss function
is rather insensitive to ``noise'', \textit{i.e.}, the standard
deviation $\delta$ in the distribution of the initial condition.

Finally, we also use different values of the parameters $\sigma_0$ and
$y_0$ in the diffusion and drift functions.  The reconstructed drift
functions $\hat{f}$ remain accurate when $\sigma_0$ and $y_0$ are
varied.  When $\sigma_0, y_0$ are small, the corresponding diffusion
and jump functions can also be accurately reconstructed; however, when
$\sigma_0, y_0$ are large, the reconstruction of the diffusion
function can be less accurate because the trajectories for training
are more sparsely distributed.  Details of the results are given
in~\ref{noise_strength}.

\end{example}

It was shown in \cite{xia2023a} that the accuracy of reconstructing a
pure-diffusion process ($\bm{\beta}=0$ in Eq.~\eqref{model_equation})
can deteriorate if trajectories for training are too sparsely
distributed (too few trajectories/too high noise).  However, we find
that prior information on the drift function in
Eq.~\eqref{model_equation} enables efficient reconstruction even if
the number of training trajectories is limited, when the temporally
decoupled squared $W_2$ method for reconstructing jump-diffusion
processes without prior information would otherwise fail.  In the next
example, we demonstrate enhanced reconstruction performance after
incorporating prior information on the drift function of
Eq.~\eqref{model_equation}, greatly improving the accuracy of
reconstructed diffusion and jump functions.

\begin{example}
\rm
\label{example2}
Consider the following 1D jump-diffusion process
\begin{equation}
  \d X_t = \alpha(X_t, t) \d t + \sigma(X_t, t) \d B_t
  + \beta(X_t, t)\d \tilde{N}_t,\,\,\,  t\in[0, T],
 \label{example2_model}
\end{equation}
where $\tilde{N}_t$ is a 1D compensated Poisson process with intensity
$t$. This model, if we set $S_t\equiv e^{X_t}$, can describe the
posited stock returns under a deterministic jump ratio
\cite{merton1976option}. To test the efficiency of our temporally
decoupled squared $W_2$-distance method, we set $\alpha\equiv
r_0=0.05$ (\textit{i.e.}, the drift function to be a constant
risk-free interest rate \cite{hull1993options}), the initial condition
$X_0 = 1$, and $T=5.1$ and explore different forms of the diffusion
and jump functions $\sigma(X, t)$ and $\beta(X, t)$.  We then input
the drift, diffusion, or the jump function $\alpha, \sigma$, or
$\beta$ in Eq.~\eqref{example2_model} as prior information to test how
well our method can reconstruct the other terms.

Summarizing, i) we first give no prior information and reconstruct all
three functions $\alpha, \sigma$, and $\beta$; ii) we specify the
risk-free interest rate $\alpha\equiv r_0$ and reconstruct $\sigma$,
and $\beta$; iii) we provide the diffusion function $\sigma$ and
reconstruct $\alpha$ and $\beta$; iv) we provide the jump function
$\beta$ and reconstruct $f$, and $\sigma$.  In this example, ``const''
refers to using a constant diffusion or jump function:
\begin{equation}
\sigma(X, t)\equiv\sigma_0\,\,\, \text{or}\,\,\, \beta(X, t)\equiv\beta_0,
\label{constant_2}
\end{equation}
``linear'' refers to using a linear diffusion or jump function
\begin{equation}
  \sigma(X, t)\equiv\sigma_0X\,\,\, \text{or}\,\,\,
  \beta(X, t)\equiv\beta_0X,
\label{linear_2}
\end{equation} 
and ``langevin'' refers to using a diffusion or jump function of the
following form
\begin{equation}
  \sigma(X, t)\equiv\sigma_0\sqrt{|X|}\,\,\, \text{or}\,\,\,
  \beta(X, t)\equiv\beta_0\sqrt{|X|}.
\label{langevin_2}
\end{equation}

\begin{figure}
   \centering
    \hspace{1cm}\includegraphics[width=0.88\linewidth]{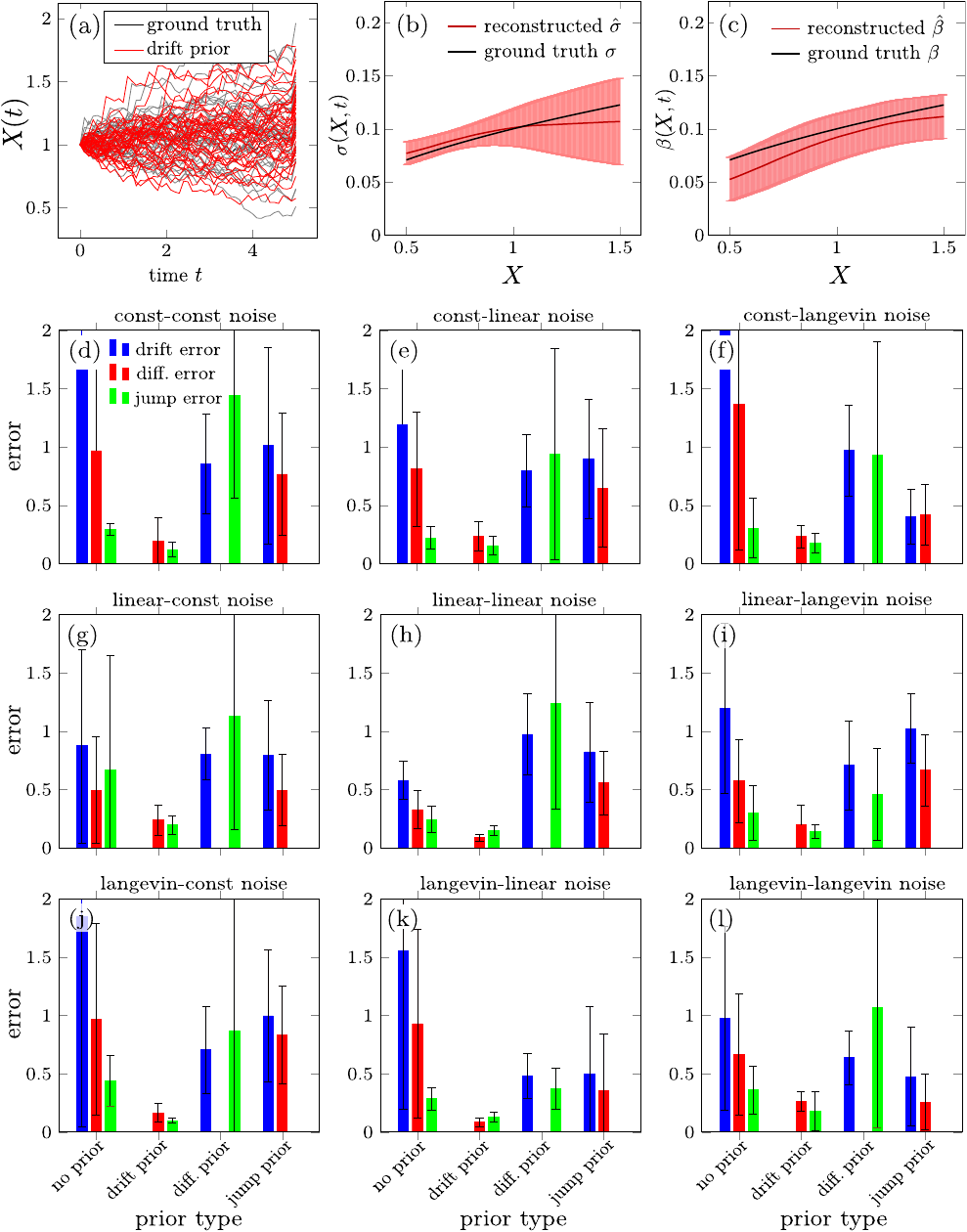}
\caption{(a) The trajectories generated by the ground truth (black)
  jump-diffusion process with $\sigma(X, t)\equiv 0.1\sqrt{|X|}~
  \text{and}~\beta(X, t)\equiv 0.1\sqrt{|X|}$ and given drift function
  in Eq.~\eqref{example2_model}, plotted against reconstructed
  trajectories (red) using the same drift function prior. (b-c) The
  ground truth diffusion and jump functions
  $\sigma(X,t)\equiv\sigma_0\sqrt{|X|}$ and $\beta(X,
  t)\equiv\beta_0\sqrt{|X|}$) shown against the reconstructed functions
  $\hat{\sigma}(X, t)$ and $\hat{\beta}(X, t)$. (with drift function
  given as prior). The red curves are the mean $\hat{\sigma}(X, t)$
  and $\hat{\beta}(X, t)$ while the shaded bands show their standard
  deviations, calculated over 5 independent experiments). (d-k) The
  reconstruction errors of the drift, diffusion, and jump functions
  without prior information on Eq.~\eqref{example2_model} or with one
  of the drift, diffusion, and jump functions given. When the drift
  function is given, errors in the reconstructed diffusion and jump
  functions are the smallest in all cases (error bars under ``drift
  prior.'')}
    \label{fig:example2}
\end{figure}

To illustrate the reconstruction, we set $\sigma_0=\beta_0=0.1$ in
Eqs.~\eqref{constant_2}, ~\eqref{linear_2}, and \eqref{langevin_2},
and plot in Fig.~\ref{fig:example2}(a) the ground truth solutions
(black) generated from Eq.~\eqref{example2_model} with a given drift
function.  Using the same drift function, trajectories of the
reconstructed jump-diffusion process are shown in red and exhibit a
distribution that matches well with that of the ground truth
solutions.  Moreover, as shown in Fig.~\ref{fig:example2}(b,c), the
differences between the learned diffusion and jump functions
$\hat{\sigma}(X, t)$ and $\hat{\beta}(X, t)$ and the ground truth
diffusion and jump functions $\sigma(X, t)=0.1\sqrt{|X|}, \beta(X,
t)=0.1\sqrt{|X|}$ are small.



If no prior information on Eq.~\eqref{example2_model} is given, the
average errors for the reconstructed drift, diffusion, and jump
functions are 1.412, 0.790, and 0.347, respectively.  This high error
might arise from training set trajectories that are too noisy or
sparsely distributed. However, if the drift function is given, the
diffusion and jump functions can be much more accurately
reconstructed, leading to relative errors below 0.2 for all three
forms of $\sigma(X_t, t)$ and $\beta(X_t, t)$ used to define the
ground truth [see Fig.~\ref{fig:example2} (d-k)]. On the other hand,
providing the diffusion or jump function does not improve the accuracy
of the reconstruction of the other unknown functions in
Eq.~\eqref{example2_model}.  The average errors of the reconstructed
diffusion and jump functions, when different prior information is
given, are listed in Table~\ref{tab:example2_errors}.

In ~\ref{num_traj}, we carry out an additional numerical experiment by
varying the number of trajectories in the training set. The errors in
the reconstructed drift, diffusion, and jump function decrease when
the number of trajectories for training increases without any prior
information. This indicates that our temporally decoupled squared
$W_2$ method has the potential to accurately reconstruct
Eq.~\eqref{example2_model} even without prior information provided
there are a sufficient number of training trajectories. On the other
hand, if the drift function is given as prior information, the errors
of the reconstructed diffusion and jump functions are around 0.2 even
when only 100 trajectories are used. Therefore, information on the
drift function can significantly boost the performance of our
temporally decoupled squared $W_2$ method, allowing accurate
reconstruction of Eq.~\eqref{example2_model} even when the number of
observed trajectories is limited.

In real physical systems, the drift function can often be obtained by
measurements over a macroscopic ensemble of trajectories, such as
mass-action kinetics if the $X(t)$ in Eq.~\eqref{model_equation} denotes
some physical quantity, \textit{e.g.}, the number density of molecules
\cite{koudriavtsev2011law,chellaboina2009modeling}. Thus, after
independently measuring the drift function and inputting it as a prior
knowledge, our temporally decoupled squared $W_2$-distance method can
be used to reconstruct the diffusion and jump functions efficiently.

\begin{table}[h!]
\centering
  \caption{Average errors in the reconstructed drift, diffusion, and
    jump functions when using the temporally decoupled squared $W_2$
    distance to reconstruct Eq.~\eqref{example2_model}. The error is
    taken over 9 possible combinations of different forms of diffusion
    and jump functions (constant, Langevin, and linear in
    Eqs.~\eqref{constant_2}, \eqref{langevin_2}, and \eqref{linear_2}) in
    Fig.~\ref{fig:example2}.}
  \vspace{2mm}
\small
\begin{tabular}{lccccc}
\toprule Prior info & error of reconstructed $\hat{f}$ & Error
of reconstructed $\hat{\sigma}$ & Error of reconstructed $\hat{\beta}$
\\ \midrule
No prior & \(1.412 (\pm 1.520) \) & \(0.790 (\pm 0.714) \) & \( 0.347
(\pm 0.356) \) \\
Given $\alpha(x, t)$ & 0 & \(0.189 (\pm 0.124) \) & \(
0.150 (\pm 0.080) \) \\ Given $\sigma(x, t)$ & \(0.771 (\pm 0.333) \)
& 0 & \( 0.939 (\pm 0.854) \) \\
Given $\beta(x, t)$ & \( 0.769 (\pm 0.520) \) & \( 0.556 (\pm 0.393)
\) & 0 \\ \bottomrule
\end{tabular}
\label{tab:example2_errors}
\end{table}

We carry out an extra numerical experiment reconstructing
Eq.~\eqref{example2_model} by varying $\sigma_0, \beta_0$ in
Eqs.~\eqref{constant_2},~\eqref{linear_2}, and~\eqref{langevin_2}. With the
drift function provided, our temporally decoupled squared
$W_2$-distance method can accurately reconstruct the diffusion and the
jump functions for different values of $\sigma_0$ and $\beta_0$ in
Eqs.~\eqref{constant_2}, \eqref{linear_2}, and \eqref{langevin_2}. The
results are shown in~\ref{varying_coef}.
\end{example}

In our last example, we test whether our temporally decoupled
squared $W_2$-distance can accurately reconstruct a 2D jump-diffusion
process with correlated Brownian-type and compensated-Poisson-type
noise across the two stochastic variables.

\begin{example}
    \rm
    \label{example3}
We reconstruct the following 2D jump-diffusion process, which is
obtained by superimposing a 2D compensated Poisson process
$\tilde{\bm{N}}_t\coloneqq (\tilde{N}_1(t), \tilde{N}_2(t))$ onto the
pure diffusion process that describes the dynamics of a synthetic data
set \cite{gao2022data,ma2021learning}:
\begin{equation}
  \d \bm{X}(t) = -\bm{g}(\bm{X}(t))\d t +
  \bm{\sigma}(\bm{X}(t))\d\bm{W}_t + \bm{\beta}(\bm{X}(t))\d
  \tilde{\bm{N}}_t, \,\,\, \bm{X}(t=0) = \bm{X}_0,\,\, t\in[0, T].
\label{example3_model}
\end{equation}
$\tilde{N}_1(t)$ and $\tilde{N}_2(t)$ are independent and both have
intensity $t$. Here, $\bm{X}(t)=(X_1(t), X_2(t))\in\mathbb{R}^2$,
$\bm{g}(\bm{X}):\mathbb{R}^2\rightarrow\mathbb{R}^2$ is the drift
function, and $\bm{\sigma},
\bm{\beta}:\mathbb{R}^2\rightarrow\mathbb{R}^{2\times2}$ are the
diffusion and jump functions, respectively. The drift function
$\bm{g}$ is given by
%
%
\begin{equation}
\begin{aligned}
  \bm{g}(\bm{X}) = & \bigg(\frac{1}{\sigma_1}\frac{N_1}{N_1+N_2}(X_1 - \mu_{11})
  +\frac{1}{\sigma_2}\frac{N_2}{N_1+N_2}(X_1 - \mu_{21}), \\
  \: &\qquad\quad    \frac{1}{\sigma_1}\frac{N_1}{N_1+N_2}(X_2 - \mu_{21})
  +\frac{1}{\sigma_2}\frac{N_2}{N_1+N_2}(X_2 - \mu_{22})\bigg)^T, \\
N_1\coloneqq & N_1(\bm{X})= \frac{1}{\sqrt{2\pi}\sigma_1}
  \exp\Big(-\tfrac{(X_1-\mu_{11})^2}{2\sigma_1^2}-
  \tfrac{(X_2-\mu_{12})^2}{2\sigma_1^2}\Big), \\
  N_2\coloneqq & N_2(\bm{X})=\frac{1}{\sqrt{2\pi}\sigma_2}
  \exp\Big(-\tfrac{(X_1-\mu_{21})^2}{2\sigma_2^2}
  -\tfrac{(X_2-\mu_{22})^2}{2\sigma_2^2}\Big).
\end{aligned}
  \label{g_definition}
\end{equation}
The parameters are set as $\sigma_1=1, \sigma_2=0.95, \mu_{11}=1.6,
\mu_{12}=1.2, \mu_{21}=1.8, \mu_{22}=1.0$. We set $T = 10.2$ and an
initial condition $\bm{X}_0=(1.7, 1.1)$.
We take the correlated diffusivity as
\begin{equation}
\bm{\sigma}=
\begin{bmatrix}
    \sigma_0\sqrt{|X_1|} & c_1\sigma_0\sqrt{|X_2|} \\
    c_1\sigma_0\sqrt{|X_1|} & \sigma_0\sqrt{|X_2|} 
\end{bmatrix},
\label{sigma2d}
\end{equation}
and the jump function of the compensated Poisson
process as
\begin{equation}
        \bm{\beta}=
\begin{bmatrix}
    \beta_0  & c_2\beta_0 \\
    c_2\beta_0 & \beta_0 \\
\end{bmatrix}.
\label{beta2d}
\end{equation}
Here, $c_1$ and $c_{2}$ determine the correlations of Brownian noise
and compensated Poisson process across the two dimensions,
respectively. Specifically, when $c_1=0$ (or $c_2=0$), the Brownian
(or compensated Poisson) noise in each variable is independent of the
other; when $c_1=1$ (or $c_2=1$), the Brownian-type (or
compensated-Poisson-type) noise across the two dimensions are linearly
dependent; when $c_1=-1$ (or $c_2=-1$), the Brownian-type (or
compensated-Poisson-type) noise across the two dimensions are
perfectly negatively correlated.
    
From Example~\ref{example2}, imposing a prior on the drift function
can greatly improve the accuracy of the reconstructed diffusion and
jump functions. Thus, we input $\bm{g}(\bm{X})$ defined in
Eq.~\eqref{g_definition} as prior information.  Since the jump-diffusion
process described by Eq.~\eqref{example3_model} is two-dimensional, we
use the following error metric to measure the errors in the
diffusion and jump functions:
\begin{equation}
\text{diffusion error} =  \frac{\sum_{i=0}^N\sum_{j=1}^{M_s}
\|\bm{\sigma}\bm{\sigma}^T(x_j(t_i), t_i) -
\hat{\bm{\sigma}}\hat{\bm{\sigma}}^T(x_j(t_i),
t_i)\|^2_F}{\sum_{i=0}^N\sum_{j=1}^{M_s}\|\hat{\bm{\sigma}}\hat{\bm{\sigma}}^T(x_{j}(t_i),
t_i)\|_F^2};
\label{sigma_error}
\end{equation}
\begin{equation}
\text{jump error}=  \frac{\sum_{i=0}^N\sum_{j=1}^{M_s}
\|\bm{\beta}\bm{\beta}^T(x_j(t_i), t_i) -
\hat{\bm{\beta}}\hat{\bm{\beta}}^T(x_j(t_i),
t_i)\|^2_F}{\sum_{i=0}^N\sum_{j=1}^{M_s}\|\hat{\bm{\beta}}\hat{\bm{\beta}}^T(x_{j}(t_i),
t_i)\|_F^2}.
\label{beta_error}
\end{equation}
Here, $\|\cdot\|_F$ denotes the Frobenius norm of a matrix. We set
$\sigma_0=0.1, \beta_0=0.1$ in Eqs.~\eqref{sigma2d} and
\eqref{beta2d}.  Different values of $c_1, c_2$ are used to tune the
correlations to explore how they affect the reconstruction of the
jump-diffusion process.
\begin{figure}
\centering
\includegraphics[width=0.98\linewidth]{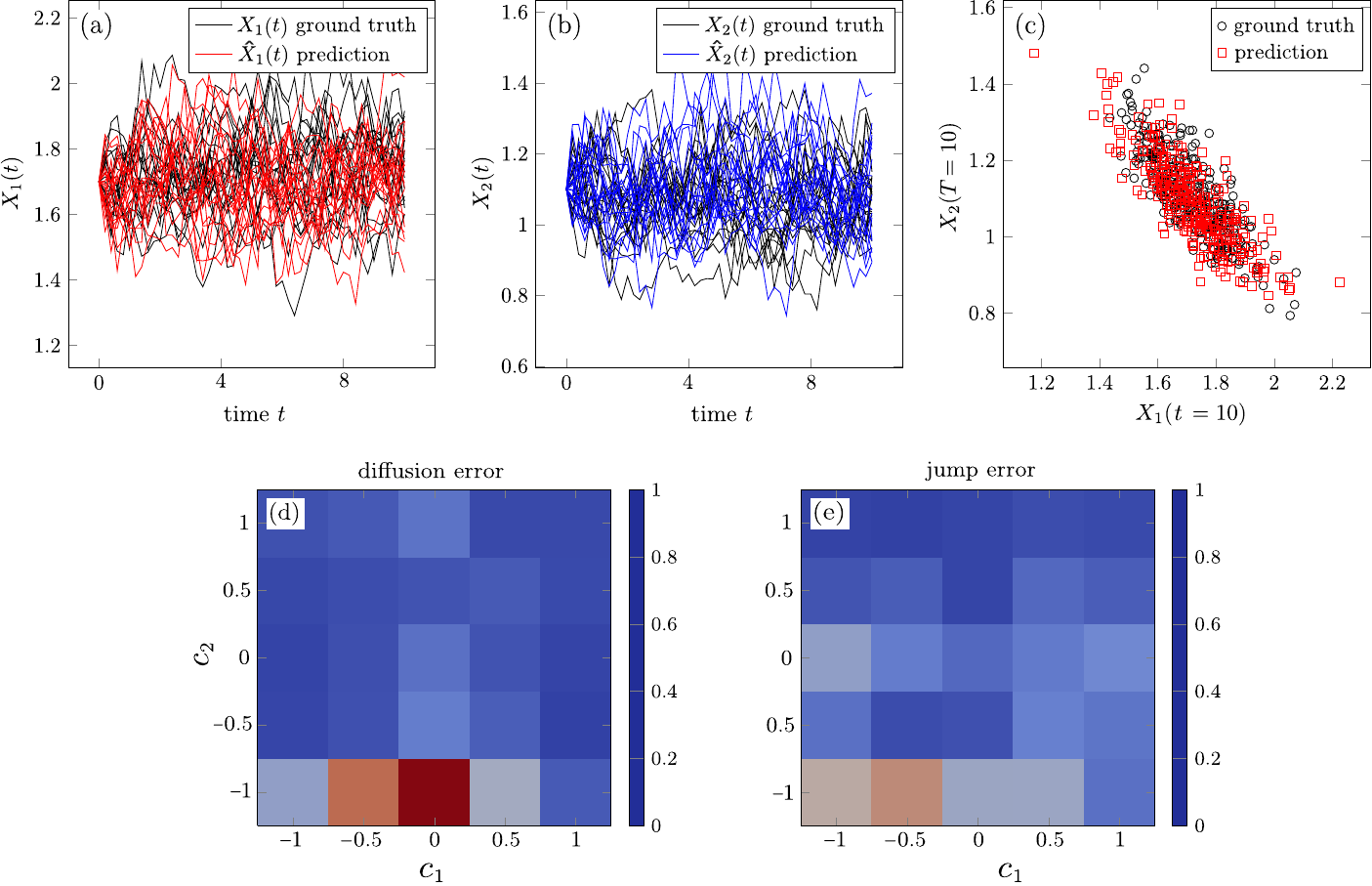}
\caption{(a-b) Solutions generated by the reconstructed jump-diffusion
  process using our temporally decoupled squared $W_2$ method versus
  solutions generated by the ground truth
  Eq.~\eqref{example3_model}. (c) The reconstructed $\hat{\bm{X}}(t=10)$
  versus the ground truth $\bm{X}(t=10)$. In (a-c), $c_1=c_2=-0.5$.
  (d) The error (Eq.~\eqref{sigma_error}) between the ground truth
  diffusion function $\bm{\sigma}$ and the reconstructed diffusion
  function $\hat{\bm{\sigma}}$. (e) The error (Eq.~\eqref{beta_error})
  between the ground truth jump function $\bm{\beta}$ and the
  reconstructed diffusion function $\hat{\bm{\beta}}$. In (d-e), the
  errors are averaged over 5 independent experiments.}
\label{fig:example3}
\end{figure}

Figs.~\ref{fig:example3}(a-c) show that solutions generated by our
reconstructed jump-diffusion process with the temporally decoupled
squared $W_2$ loss function match well with solutions generated by the
2D jump-diffusion process Eq.~\eqref{example3_model} ($c_1=c_2=-0.5$ in
Eqs.~\eqref{sigma2d} and \eqref{beta2d}).  Figs.~\eqref{fig:example3}(d-e)
indicate that when the drift function $\bm{g}(\bm{X}(t))$ is given,
our temporally decoupled squared $W_2$ method can accurately
reconstruct the diffusion and jump functions for most combinations of
$c_1, c_2$ The average errors in the diffusion and jump functions
averaged over all combinations of $c_1,c_2$ are 0.197 and 0.210,
respectively. Also, the final distribution of the reconstructed
$\hat{\bm{X}}(t)$ aligns well with the ground truth $\bm{X}(t)$.

In ~\ref{number_layers}, we implement our reconstruction method by
using different numbers of hidden layers and different numbers of
neurons in each layer for the neural-network-parameterized
approximation to the diffusion and jump functions $\bm{\sigma}$ and
$\bm{\beta}$. We find that with the drift function given as prior information,
increasing the number of neurons per layer can improve the accuracy of
the reconstructed diffusion and the jump function of the 2D
jump-diffusion process Eq.~\eqref{example3_model}. Increasing the number
of hidden layers also leads to a more accurate reconstruction of the
diffusion and jump function when the number of hidden layers is
smaller than three; however, after three hidden layers, increasing
their number leads to less accuracy of the reconstructed $\bm{\sigma}$
and $\bm{\beta}$.  Setting the number of hidden layers to three and
the number of neurons per layer to about 400 leads to excellent
reconstruction of $\bm{\sigma}$ and $\bm{\beta}$.  However, larger
numbers of hidden layers or neurons per hidden layer demand more
memory usage and lead to longer runtimes. The optimal network
architecture for reconstructing the diffusion and jump functions is
worth further investigation.

\end{example}

\section{Summary \& conclusions}
\label{summary}
In this paper, we proposed and analyzed Wasserstein-distance-based
loss functions for reconstructing jump-diffusion
processes. Specifically, we showed that a temporally decoupled squared
$W_2$-distance $\tilde{W}_2^2(\mu, \hat{\mu})$ defined in
Eq.~\eqref{decoupled_def} provides both upper and lower bounds for
errors in the drift, diffusion, and jump functions
$\bm{f}-\hat{\bm{f}}$, $\bm{\sigma}-\hat{\bm{\sigma}}$, and
$\bm{\beta}-\hat{\bm{\beta}}$ when approximating a jump-diffusion
process Eq.~\eqref{model_equation} by another jump-diffusion process
\ref{approximate_equation}. Moreover, the temporally decoupled squared
$W_2$-distance can be efficiently evaluated using finite-sample
finite-time-point observations, which yields an easy-to-calculate loss
function for reconstructing jump-diffusion processes.

Through several numerical experiments, we showed that minimizing our
proposed temporally decoupled squared $W_2$-distance loss performs
much better than other commonly used loss functions and methods
 for jump-diffusion process reconstruction using
parameterized neural networks.  Furthermore, we showed that if we imposed
prior knowledge on the drift function of the jump-diffusion process,
then the diffusion and jump functions could be more accurately
reconstructed.

Although we analyzed a 2D stochastic jump-diffusion process,
investigating how our temporally decoupled squared $W_2$ method
performs in higher dimensional jump-diffusion processes with more general noise
correlation matrices should be further explored. Such analyses would
inform how our approach can be directly applied to real-world
multi-dimensional jump-diffusion models with applications in finance,
biology, physics, and other disciplines. Similarly, how the
Wasserstein distance can be adapted to reconstruct other stochastic
processes such as L\'evy walks involving compound Poisson process
\cite{bertoin1996levy,barndorff2001levy} is a potentially fruitful
area of further investigation. Reconstructing such processes could
require inferring the intensity of the Poisson process, which is
nontrivial and would require consideration of differentiation
w.r.t. ``discrete randomness'' \cite{arya2022automatic}.


\section*{Data availability statement}
No data was used during this study. All code will be published upon
acceptance of this manuscript.


\section*{Conflict of interest}
The authors declare that they have no conflicts of interest to report
regarding the present study.


\section*{References}
\bibliographystyle{unsrt.bst}
\bibliography{bibliography}

\newpage
\appendix
\section{Proof to Theorem~\ref{theorem1}}
\label{proof_theorem1}
Here, we provide proof for Theorem~\ref{theorem1}. Our strategy is
similar to that used in the proof of the stochastic Gronwall lemma
(Theorem 2.2 in \cite{mehri2019stochastic}). First, we apply the Ito's
lemma to

\begin{equation}
  \begin{aligned}
    \big|X_i(t)-\tilde{X}_i(t)\big|^2 = & 2\int_0^t\big( X_i(s^-)
    - \tilde{X}_i(s^-)\big)\big(f_i(\bm{X}(s^-), s^-)
    - \hat{f}_i(\tilde{\bm{X}}(s^-), s^-)\big)\d s \\
    \: & +2\int_0^t\sum_{j=1}^m\big(X_i(s^-)
    - \tilde{X}_i(s^-)\big)\big(\sigma_{i, j}(\bm{X}(s^-), s^-)
    - \hat{\sigma}_{i, j}(\tilde{\bm{X}}_i(s^-), s^-)\big)\d B_{j, s}\\
    \: & +2\int_0^t\int_U \big(X_i(s^-) - \tilde{X}_i(s^-)\big)
    \big(\beta_{i}(\bm{X}(s^-), \xi, s^-)
    - \hat{\beta}_{i}(\hat{\bm{X}}(s^-), \xi, s^-)\big)\d \tilde{N}
    \big( \d s, \nu(\d\xi)\big) \\
    \: & + \int_0^t\sum_{j=1}^m \big(\sigma_{i, j}(X_i(s^-), s^-)
    -\hat{\sigma}_{i, j}(\tilde{X}_i(s^-), s^-)\big)^2\d s \\
    \: & +  \int_0^t\int_U  \big(\beta_{i}(X_i(s^-), \xi, s^-)
    -\hat{\beta}_{i, j}(\tilde{X}_i(s^-), \xi, s^-)\big)^2\nu(\d \xi)\d s.
    \end{aligned}
\label{x_representation}
\end{equation}
Note that

\begin{equation}
\begin{aligned}
 f_i(\bm{X}(s^-), s^-) - \hat{f}_i(\tilde{\bm{X}}(s^-), s^-) = &
    \big(f_i(\bm{X}(s^-), s^-)-\hat{f}_i(\bm{X}(s^-), s^-)\big) \\[-1pt]
    \: & \qquad\quad + \big(\hat{f}_i(\bm{X}(s^-), s^-)
    - \hat{f}_i(\tilde{\bm{X}}(s^-), s^-)\big), \\[3pt]
    \sigma_{i,j}(\bm{X}(s^-), s^-) - \hat{\sigma}_{i, j}(\tilde{\bm{X}}(s^-), s^-) = &
    \big(\sigma_{i,j}(\bm{X}(s^-), s^-)
    - \hat{\sigma}_{i,j}(\bm{X}(s^-), s^-)\big) \\[-1pt]
    \: & \qquad\quad + \big(\hat{\sigma}_{i,j}(\bm{X}(s^-), s^-)
    - \hat{\sigma}_{i, j}(\tilde{\bm{X}}(s^-), s^-)\big), \\[3pt]
    \beta_i(\bm{X}(s^-), \xi, s^-)-\hat{\beta}_i(\tilde{\bm{X}}(s^-), \xi, s^-)= &
    \big(\beta_i(\bm{X}(s^-), \xi, s^-) - \hat{\beta}_i(\bm{X}(s^-), \xi, s^-)\big) \\[-1pt]
    \: & \qquad\quad + \big(\hat{\beta}_i(\bm{X}(s^-), \xi, s^-)
    - \hat{\beta}_{i}(\tilde{\bm{X}}(s^-), \xi, s^-)\big).
    \end{aligned}
    \label{h_def}
\end{equation}
Using the Lipschitz conditions on the drift, diffusion, and jump
functions $\hat{\bm{f}}, \hat{\bm{\sigma}}$, and $\hat{\bm{\beta}}$ in Assumption~\ref{assumptions} as well as the Cauchy inequality, from Eq.~\eqref{x_representation}
and~\ref{h_def}, we find
\begin{equation}
  \begin{aligned}
    \hspace{-0.25cm}\big|\bm{X}(t)-\tilde{\bm{X}}(t)\big|_2^2 &  = \sum_{i=1}^d
    \big(X_i(t)-\hat{X}_i(t)\big)^2 \\
    \: &  \leq  \big(2\overline{f}+ 2\overline{\sigma} m + 2\nu(U)\overline{\beta} + 1+m + \nu(U)\big)\sum_{i=1}^d\int_0^t\big( X_i(s^-)
    - \tilde{X}_i(s^-)\big)^2\d s \\
    \: & \quad +\sum_{i=1}^d\int_0^t \big|f_i(\bm{X}(s^-),s^-)
    -\hat{f}_i\bm{X}(s^-), s^-)\big|^2\d s \\
    \: & \quad +2\sum_{i=1}^d\int_0^t\sum_{j=1}^m
    \big|\sigma_{i, j}(\bm{X}(s^-), s^-)
    - \hat{\sigma}_{i, j}(\bm{X}(s^-), s^-)\big|^2\d s \\
    \: & \quad +2\sum_{i=1}^d\int_0^t\int_U\big|\beta_{i}(\bm{X}(s^-), \xi, s^-)
    - \hat{\beta}_{i}(\bm{X}(s^-),\xi, s^-)\big|^2\nu(\d\xi)\d s \\
    \: & \quad +2\sum_{i=1}^d\int_0^t\sum_{j=1}^m\big( X_i(s^-)
    - \tilde{X}_i(s^-)\big)\big(\sigma_{i, j}(\bm{X}(s^-), s^-)
    - \hat{\sigma}_{i, j}(\tilde{\bm{X}}_i(s^-), s^-)\big)\d B_{j, s} \\
    \: & \quad +2\sum_{i=1}^d\!\int_0^t\!\int_U\big( X_i(s^-)
    - \tilde{X}_i(s^-)\big) \big(\beta_{i}(\bm{X}(s^-), \xi, s^-)
    - \hat{\beta}_{i}(\tilde{\bm{X}}(s^-), \xi, s^-)\big)
    \d \tilde{N}\big(\d s, \nu(\d\xi)\big).
    \end{aligned}
    \label{X_bound}
\end{equation}

From Assumption~\ref{assumptions} and the conditions in
Theorem~\ref{theorem1}, the second, third, and fourth terms on the RHS
of Eq.~\eqref{X_bound} are adapted and non-decreasing w.r.t. $t$; the
fifth and sixth terms on the RHS of Eq.~\eqref{X_bound} are
martingales. Thus, by taking the expectation of both sides of
Eq.~\eqref{X_bound}, we find
\begin{equation}
  \E\Big[\big|\bm{X}(t)-\tilde{\bm{X}}(t)\big|_2^2\Big]
  \leq \big(2\overline{f}+1 + (2\overline{\sigma}+1) m + (2\overline{\beta}+1)\nu(U) \big)
  \int_0^t\E\Big[\big|\bm{X}(s)-\tilde{\bm{X}}(s)\big|_2^2\Big]\d s + \E[H(t)],
\end{equation}
where $H(t)$ is defined in Eq.~\eqref{h_define}.
Applying Gronwall's lemma to $u(t) \coloneqq
\E\Big[\big|\bm{X}(t)-\tilde{\bm{X}}(t)\big|_2^2\Big]$ and noticing
that $\E[H(t)]$ is non-decreasing w.r.t. $t$, we conclude that
\begin{equation}
  u(t)\leq \E\Big[\big|\bm{X}(t) -
    \tilde{\bm{X}}(t)\big|^2\,\Big|\,\bm{X}(0)\Big]
  \leq \exp\Big( (2\overline{f}+1 + (2\overline{\sigma}+1) m + (2\overline{\beta}+1)\nu(E) )T\Big)\cdot  \E\big[H(T)\, |\, \bm{X}(0)\big],
\end{equation}
which proves Eq.~\eqref{pleq2}.

\section{Proof to Theorem~\ref{theorem3}}
\label{theorem3proof}

Here, we shall provide proof of Theorem~\ref{theorem3}, which
generalizes Theorem~2 in \cite{xia2023a} for pure diffusion processes.
Denote
\begin{equation}
  \Omega_N\coloneqq \{\bm{Y}(t)| \bm{Y}(t) = \bm{Y}(t_i)\,\,\, t\in [t_i, t_{i+1}), i<N-1;
    \,\,\bm{Y}(t) = \bm{Y}(t_i), \,\,\,t\in[t_i, t_{i+1}]\}
\end{equation}
to be the space of piecewise functions. Clearly, it is a subspace of
$L^2([0, T]; \mathbb{R}^d)$. Also, the embedding map from $\Omega_N$
to $L^2([0, T]; \mathbb{R}^d)$ preserves the $\|\cdot\|$ norm, which
enables us to define the measures on $\mathcal{B}(L^2([0, T];
\mathbb{R}^d))$ induced by the measures $\mu_N, \hat{\mu}_N$. For
simplicity, we shall still denote those induced measures by $\mu_N,
\hat{\mu}_N$.

Suppose $\bm{X}(t), \hat{\bm{X}}(t)$ are generated by two
jump-diffusion processes defined by Eq.~\eqref{model_equation} and
Eq.~\eqref{approximate_equation}. The inequality~Eq.~\eqref{triangular} is
a direct result of the triangular inequality for the Wasserstein
distance \cite{clement2008elementary} because $\bm{X}, \bm{X}_N,
\hat{\bm{X}}, \hat{\bm{X}}_N\in L^2([0, T]; \mathbb{R}^d)$.

Next, we prove Eq.~\eqref{dtbound}. Because $\bm{X}_N(t)=I_N \bm{X}(t)$
(defined in Eq.~\eqref{X_N_def}), we choose a specific
\textit{coupling measure}, i.e. the coupled distribution, $\pi$ of $\mu,
\mu_N$ that is essentially the ``original'' probability distribution.
To be more specific, for an abstract probability space $(\Omega,
\mathcal{A}, p)$ associated with ${\boldsymbol X}$, $\mu$ and
$\mu_{N}$ can be characterized by the \textit{pushforward} of $p$ via
${\boldsymbol X}$ and ${\boldsymbol X}_N$
respectively, i.e., $\mu = {\boldsymbol X}_* p$, defined by $\forall A
\in \mathcal{B}\big(\tilde{\Omega}_N\big)$, elements in the Borel
$\sigma$-algebra of $\tilde{\Omega}_N$,

\begin{equation}
\mu(A) = {\boldsymbol{X}}_*p(A) \coloneqq p\big({\boldsymbol X}^{-1}(A)\big),
\end{equation}
where $\boldsymbol X$ is interpreted as a measurable map from $\Omega$
to $\tilde{\Omega}_N$, and $\boldsymbol{X}^{-1}(A)$ is the preimage of
$A$ under $\boldsymbol{X}$.  Then, the coupling $\pi$ is defined by

\begin{equation}
\pi = ({\boldsymbol X}, {\boldsymbol X}_N)_* p,
\end{equation}
where $({\boldsymbol X}, {\boldsymbol X}_N)$ are interpreted as a
measurable map from $\Omega$ to $\tilde{\Omega}_N\times
\tilde{\Omega}_N$.  One can readily verify that the marginal
distributions of $\pi$ are $\mu$ and $\mu_{N}$ respectively.
Therefore, the squared $W_2^2(\mu, \mu_N)$ can be bounded by
\begin{equation}
  W_2^2(\mu, \mu_N)\leq \sum_{i=1}^N \int_{t_{i-1}}^{t_i}
  \!\E\Big[\big|\bm{X}(t) - \bm{X}_N(t)\big|_2^2\Big]\d t =
  \sum_{i=1}^N \int_{t_{i-1}}^{t_i} \sum_{\ell=1}^d
  \E\Big[\big(X_{\ell}(t) - X_{N, \ell}(t)\big)^2\Big]\d t
\end{equation}
For each $\ell=1,...,d$, by using the It\^o's isometry and the
orthogonality condition of the compensated Poisson process $\tilde{N}$
(in Assumption~\ref{assumptions}), we have
\begin{equation}
  \begin{aligned}
   \sum_{i=1}^N \int_{t_{i-1}}^{t_i}  \!
    \E\big[\big(X_{\ell}(t)-X_{N, \ell}(t)\big)^2\big]\d t  \leq &
 \sum_{i=1}^N\int_{t_{i-1}}^{t_i}
 \E\bigg[\Big(\medint\int_{t_i}^t f_{\ell}(\bm{X}(r^-), r^-)
   \d r\Big)^2\bigg]\text{d}t \\
\: & +  \sum_{i=1}^N\int_{t_{i-1}}^{t_i} \E\bigg[\Big(\medint\int_{t_i}^t
  \sum_{j=1}^m\sigma_{\ell, j}(\hat{X}(r^-), r^-)\d B_{j, r}\Big)^2\bigg]\text{d}t \\
\: & + \sum_{i=1}^N\int_{t_{i-1}}^{t_i} \E\bigg[\Big(\medint\int_{t_i}^t
  \medint\int_U \beta_{\ell}(\bm{X}(r^-), \xi, r^-)\tilde{N}(\d r, \nu(\d\xi)) \Big)^2\bigg]
\text{d}t \\
\leq & \sum_{i=1}^N \big(\Delta t_{i-1}\big)^2\E\Big[
  \medint\int_{t_{i-1}}^{t_i}\!f_{\ell}^2\d t\Big]
+ \sum_{i=1}^N \big(\Delta t_{i-1}\big)^{2}\sum_{j=1}^m\E
\Big[\medint\int_{t_{i-1}}^{t_i}\!\sigma_{\ell, j}^2\d t\Big] \\
\: & + \sum_{i=1}^N\Delta t_{i-1}\E \Big[\medint\int_{t_{i-1}}^{t_i}
  \medint\int_U\beta_{\ell}^2\nu(\d\xi)\d t\Big],
\end{aligned}
\end{equation}
where $\Delta t_{i-1} \coloneqq t_{i}-t_{i-1}$. Summing over $\ell$,
we have
\begin{equation}
  \sqrt{\sum_{i=1}^N \medint\int_{t_{i-1}}^{t_i} \E\Big[\big|\bm{X}(t)
    - \bm{X}_N(t)\big|_2^2\Big]\d t } \leq
 \sqrt{F\Delta t^2 + \Sigma\Delta t + B\Delta t},
\label{X_piecebound}
\end{equation}
where $\Delta t\coloneqq\max_{0\leq i\leq N-1}(t_{i+1}-t_i)$.
Similarly, we can show that 

\begin{equation}
  W_2(\hat{\mu}, \hat{\mu}_N)\leq \sqrt{\hat{F}\Delta t^2
  + {\hat{\Sigma}}\Delta t+ {\hat{B}}\Delta t}.
\label{Xhat_piecebound}
\end{equation}
Plugging~Eq.~\eqref{X_piecebound} and~Eq.~\eqref{Xhat_piecebound}
into~Eq.~\eqref{triangular}, we have proved~Eq.~\eqref{dtbound}.

\section{Proof to Theorem~\ref{theorem4}}
\label{theorem4proof}
Here, we provide proof to Theorem~\ref{theorem4}. The proof builds
upon and generalizes the proof of Theorem~3 in \cite{xia2023a} for
pure diffusion processes to jump-diffusion processes.  First, notice
that

\begin{equation}
  \E\Big[\big|\bm{X}(t) - \hat{\bm{X}}(t)\big|_2^2\Big]\leq
  2(FT+\hat{F}T+\Sigma+\hat{\Sigma} + B + \hat{B})<\infty, \,\,\,\forall t\in[0, T]
\label{M_condition}
\end{equation}
where $F, \hat{F}, \Sigma, \hat{\Sigma}, B, \hat{B}$ are defined in
Eq.~\eqref{F_Sigma}.  By applying Theorem~\ref{theorem3}, for any $t_i,
i=1,2,...,N$, denoting $\Delta t_i\coloneqq t_i-t_{i-1}$ , we have

\begin{equation}
  \begin{aligned}
    \inf_{\pi_i}\sqrt{\E_{\pi_i}\big[|\bm{X}(t_i) -
        \hat{\bm{X}}(t_i)|_2^2\big]\Delta t_i} &
    - \sqrt{F_i(\Delta t_i)^2+ \Sigma_i\Delta t_i+B_i\Delta t_i}
    - \sqrt{\hat{F}_i(\Delta t_i)^2 + \hat{\Sigma}_i\Delta t_i + \hat{B}_i\Delta t_i} \\
    \: & \leq  W_2\big(\boldsymbol{\mu}_i, \hat{\boldsymbol{\mu}}_i\big) \\
    \: & \leq \inf_{\pi_i}\sqrt{\E_{\pi_i}\big[|\bm{X}(t_i)
      - \hat{\bm{X}}(t_i)|_2^2\big]\Delta t_i}
    +\!\sqrt{F_i(\Delta t_i)^2 + \Sigma_i\Delta t_i + B_i\Delta t_i} \\[-3pt]
    \: & \hspace{4.2cm}  +\!\sqrt{\hat{F}_i(\Delta t_i)^2 + \hat{\Sigma}_i\Delta t_i + \hat{B}_i\Delta t_i},
  \end{aligned}
  \label{approx_i}
\end{equation}
where $\bm{\mu}_i, \hat{\bm{\mu}}_i$ are the distributions for
$\bm{X}(t), t\in[t_{i}, t_{i+1})$ and $\hat{\bm{X}}(t), t\in[t_i,
    t_{i+1})$, respectively. Additionally, from
    Eq.~\eqref{convergence_result}, we have
    
    \begin{equation}
      \begin{aligned}
        \sum_{i=0}^{N-1}F_i = & F<\infty,\,\,\,\sum_{i=0}^{N-1}\Sigma_i
        = \Sigma<\infty,\,\,\sum_{i=0}^{N-1}B_i = B<\infty \\
        \sum_{i=0}^{N-1}\hat{F}_i = & \hat{F}<\infty,  \,\,\,
        \sum_{i=0}^{N-1}\hat{\Sigma}_i = \hat{\Sigma}<\infty,\,\,\,
        \sum_{i=0}^{N-1}\hat{B}_i = \hat{B}<\infty.
        \end{aligned}
    \end{equation}
From the inequality~\ref{approx_i}, we have
\begin{equation}
\begin{aligned}
 W_2^2(\boldsymbol{\mu}_i, \hat{\boldsymbol{\mu}}_i)\leq &
 \inf_{\pi_i}\E_{\pi_i}\big[|\bm{X}(t_i) -
   \hat{\bm{X}}(t_i)|_2^2\big] \Delta t_i \\[-4pt]
  \: & \,\, + 2\inf_{\pi_i}\sqrt{\E_{\pi_i}\big[|\bm{X}(t_i)
    - \hat{\bm{X}}(t_i)|_2^2\big]}\,
  \Delta t_i \Big[\sqrt{F_i\Delta t_i + \Sigma_i + B_i}
    +\!\sqrt{\hat{F}\Delta t_i + \hat{\Sigma}_i + \hat{B}_i}\Big] \\
  \: & \qquad  +2 \Delta t_i \Big(F_i \Delta t_i
  + \Sigma_i + B_i + \hat{F}_i \Delta t_i +\hat{\Sigma}_i  + \hat{B}_i\Big)\\[8pt]
  W_2^2(\boldsymbol{\mu}_i, \hat{\boldsymbol{\mu}}_i)\geq &
  \inf_{\pi_i}\E_{\pi_i}\big[|\bm{X}(t_i) -
    \hat{\bm{X}}(t_i)|_2^2\big]\Delta t_i \\[-3pt]
  \: & \quad - 2W_2(\boldsymbol{\mu}_i, \hat{\boldsymbol{\mu}}_i)
  \Delta t_i\Big[
  \sqrt{F_i\Delta t_i + \Sigma_i + B_i} +\!\sqrt{\hat{F}\Delta t_i
    + \hat{\Sigma}_i + \hat{B}_i} \Big] \\
\: & \qquad\, -  2\Delta t_i \Big(F_i \Delta t_i + \Sigma_i + B_i +
\hat{F}_i \Delta t_i +\hat{\Sigma}_i  + \hat{B}_i\Big)
\end{aligned}
\label{both_side}
\end{equation}
Specifically, from the assumption given in Eq.~\eqref{M_condition} and
Eq.~\eqref{approx_i}, we conclude that

\begin{equation}
  W_2(\boldsymbol{\mu}_i, \hat{\boldsymbol{\mu}}_i)\leq \sqrt{\Delta t_i}
  \Big(M+\sqrt{F\Delta t_i+\Sigma_i+B_i} +\sqrt{\hat{F}\Delta t_i
    + \hat{\Sigma}_i+\hat{B}_i}\Big) \coloneqq \tilde{M}\sqrt{\Delta t_i},
  \,\,\,\tilde{M}<\infty.
\end{equation}
Summing over $i=1,...,N-1$ for both inequalities in
Eq.~\eqref{both_side} and noting that $\Delta t=\max_i |t_{i+1}-t_i|$,
we conclude

\begin{equation}
\begin{aligned}
\sum_{i=0}^{N-1}W_2^2(\boldsymbol{\mu}_i, \hat{\boldsymbol{\mu}}_i)\leq &
\sum_{i=0}^{N-1}\inf_{\pi_i}\E_{\pi_i}\Big[\big|\bm{X}(t_i)
  - \hat{\bm{X}}(t_i)\big|_2^2\Big]\Delta t_i + 2\Delta t
\big(F\Delta t + \Sigma + \hat{F}\Delta t + \hat{\Sigma}
+ B + \hat{B}\big) \\[-6pt]
\: & \qquad\quad + 2M\sum_{i=1}^{N-1}\Delta t_i
\Big(\sqrt{F_i\Delta t_i + \Sigma_i+B_i}
+\!\sqrt{\hat{F}_i\Delta t_i + \hat{\Sigma}_i+\hat{B}_i} \Big), \\
\leq & \sum_{i=0}^{N-1}\inf_{\pi_i}\E_{\pi_i}\Big[\big|\bm{X}(t_i)
  - \hat{\bm{X}}(t_i)\big|_2^2\Big]\Delta t_i + 2\Delta t
\Big(F\Delta t + \Sigma + \hat{F}\Delta t
+ \hat{\Sigma}+ B + \hat{B}\Big) \\[-4pt]
\: & \qquad\quad + M\sqrt{\Delta t}\,\Big(\big(F+\hat{F}\big)\Delta t
  + \Sigma+\hat{\Sigma} + B+\hat{B} + 2T\Big)
\end{aligned}
\label{approx_i1}
\end{equation}
and

\begin{equation}
  \begin{aligned}
    \sum_{i=0}^{N-1}W_2^2(\boldsymbol{\mu}_i, \hat{\boldsymbol{\mu}}_i)\geq &
    \sum_{i=0}^{N-1}\inf_{\pi_i}\E_{\pi_i}\Big[\big|\bm{X}(t_i)
      - \hat{\bm{X}}(t_i)\big|_2^2\Big]
    \Delta t_i \\[-6pt]
    \: & \quad -2\tilde{M}\sum_{i=0}^{N-1}\Delta t_i
    \Big(\sqrt{F_i\Delta t_i + \Sigma_i+B_i} + \sqrt{\hat{F}_i\Delta t_i
      + \hat{\Sigma}_i+\hat{B}_i} \Big) \\
    \: & \qquad -2\Delta t
    (F\Delta t + \Sigma + B+\hat{F}\Delta t + \hat{\Sigma}+\hat{B}), \\[4pt]
    \geq & \sum_{i=0}^{N-1}\inf_{\pi_i}\E_{\pi_i}\Big[\big|\bm{X}(t_i) - \hat{\bm{X}}(t_i)\big|_2^2\Big]
    \Delta t_i - 2\Delta t \big(F\Delta t + \Sigma +B+ \hat{F}\Delta t
    + \hat{\Sigma}+\hat{B}\big)\\[-3pt]
    \: & \qquad - \tilde{M}\sqrt{\Delta t}\,\Big((F+\hat{F})\Delta t
    + \Sigma+\hat{\Sigma} + B + \hat{B} + 2T\Big)
    \end{aligned}
\label{approx_i2}
\end{equation}
Eqs.~\eqref{approx_i1} and \eqref{approx_i2} indicate that as
$N\rightarrow\infty$,

\begin{equation}
  \sum_{i=0}^{N-1}\inf_{\pi_i}\E_{\pi_i}\Big[\big|\bm{X}(t_i)
    - \hat{\bm{X}}(t_i)\big|_2^2\Big]\Delta t_i
  -  \sum_{i=0}^{N-1}W_2^2(\boldsymbol{\mu}_i, \hat{\boldsymbol{\mu}}_i)
  \rightarrow 0,
\end{equation}
which proves Eq.~\eqref{lim_condition}.

Now, suppose $0=t_0^1<t_1^1<...<t_{N_1}^1=T$;
$0=t_0^2<t_1^2<...<t_{N_2}^2=T$ to be two sets of grids on $[0,
  T]$. We define a third set of grids $0=t_0^3<...<t_{N_3}^3=T$ such
that $\{t_0^1,...,t_{N_1}^1\}\cup \{t_0^2,...,t_{N_2}^2\} =
\{t_0^3,...,t_{N_3}^3\}$. Let $\delta
t\coloneqq\max\{\max_i(t_{i+1}^1-t_i^1), \max_j(t_{j+1}^2-t_j^2),
\max_k(t_{k+1}^3-t_k^3)\}$. We denote $\mu_i^s(t_i^1)$ and
$\hat{\mu}_i^s (t_i^1)$ to be the probability distribution of
$\bm{X}(t_i^s)$ and $\hat{\bm{X}}(t_i^s)$, $s=1,2,3$, respectively.
We now prove that

\begin{equation}
\bigg|\sum_{i=0}^{N_1-1}W_2^2\big(\mu(t_i^1),
\hat{\mu}(t_i^1)\big)(t_{i+1}^1-t_i^1) -
\sum_{i=0}^{N_3-1}W_2^2\big(\mu(t_i^3),
\hat{\mu}(t_i^3)\big)(t_{i+1}^3-t_i^3)\bigg| \rightarrow 0,
\label{limit_exist}
\end{equation}
as $\Delta t\rightarrow 0$.

First, suppose in the interval $(t_i^1, t_{i+1}^1)$, we have
$t_i^1=t_{\ell}^3<t_{\ell+1}^3<...<t_{\ell+s}^3=t_{i+1}^1, s\geq 1$,
then for $s>1$, since
$t_{i+1}^1-t_i^1=\sum_{k=\ell}^{\ell+s-1}(t_{k+1}^3-t_k^3)$, we have

\begin{equation}
  \begin{aligned}
    \bigg|W_2^2\big(\mu(t_i^1), \hat{\mu}(t_i^1)\big)& (t_{i+1}^1-t_i^1)  -
    \sum_{k=\ell}^{\ell+s-1} W_2^2\big(\mu(t_k^3)),
    \hat{\mu}(t_i^3)\big)(t_{k+1}^3-t_k^3)\bigg| \\
    \: & \leq\sum_{k=\ell+1}^{\ell+s-1}\Big(W_2\big(\mu(t_i^1),
    \hat{\mu}(t_i^1)\big) + W_2\big(\hat{\mu}(t_i^3), \hat{\mu}(t_k^3)\big)\Big)  \\
    \: & \qquad\quad \times \Big(W_2\big(\mu(t_i^1), \hat{\mu}(t_i^1)\big)
    - W_2\big(\mu(t_k^3), \hat{\mu}(t_k^3)\big)\Big) (t_{k+1}^3 - t_k^3).
    \end{aligned}
\label{triang}
\end{equation}
On the other hand,  because we can take a specific coupling $\pi^*$ to be the
joint distribution of $(\bm{X}(t_i^1), \bm{X}(t_k^3))$,

\begin{equation}
  \begin{aligned}
    W_2(\mu(t_i^1), \mu(t_k^3)) \leq &
    \sqrt{\E\big[|\bm{X}(t_k^3) - \bm{X}(t_i^1)|_2^2\big]} \\
    \leq & \E\bigg[\medint\int_{t_i^1}^{t_{i+1}^1}\!
      \sum_{i=1}^d f_i^2(\bm{X}(t^-),t^-) \d t+ \medint\int_{t_i^1}^{t_{i+1}^1}\sum_{\ell=1}^d
  \sum_{j=1}^m\sigma_{\ell, j}^2(\bm{X}(t^-),t^-)\d t \\
\: & \hspace{4cm} +\!
\medint\int_{t_i^1}^{t_{i+1}^1}\sum_{\ell=1}^d\int_U \beta_{\ell}^2(\bm{X}(t^-), \xi,t^-)
\nu(\d\xi)\d t \bigg]^{\frac{1}{2}}.
\end{aligned}
\label{bound1}
\end{equation}
Similarly, we have

\begin{equation}
  \begin{aligned}
    W_2\big(\hat{\mu}(t_i^1), \hat{\mu}(t_k^3)\big) \leq &
    \,\E\bigg[\medint \int_{t_i}^{t_{i+1}} \sum_{\ell=1}^d
      \hat{f}_{\ell}^2(\bm{X}(t^-),t^-) \d t + \medint\int_{t_i^1}^{t_{i+1}^1}\sum_{\ell=1}^d
\sum_{j=1}^m\hat{\sigma}_{\ell, j}^2(\bm{X}(t^-),t^-)\d t  \\
\: & \hspace{4cm} + \medint\int_{t_i^1}^{t_{i+1}^1}\sum_{\ell=1}^d\medint\int_U 
\hat{\beta}_{\ell}^2(\hat{\bm{X}}(t^-), \xi,t^-)\nu(\d\xi)\d t \bigg]^{\frac{1}{2}}
\end{aligned}
    \label{bound2}
  \end{equation}
Using the triangular inequality of the Wasserstein distance as well as the Cauchy inequality, we have
\begin{equation}
\begin{aligned}
    \Big|W_2\big(\mu(t_i^1), \hat{\mu}(t_i^1)\big)- W_2\big(\mu(t_k^3),
\hat{\mu}(t_k^3)\big)\Big|
&\leq \Big|W_2\big(\mu(t_i^1), \hat{\mu}(t_i^1)\big)- W_2\big(\mu(t_k^3),
\hat{\mu}(t_k^1)\big)\Big| \\
&\quad\quad+ \Big|W_2\big(\mu(t_i^3), \hat{\mu}(t_i^1)\big)- W_2\big(\mu(t_k^3),
\hat{\mu}(t_k^3)\big)\Big|\\
&\leq W_2\big(\mu(t_i^1), \mu(t_k^3)\big) +
W_2\big(\hat{\mu}(t_i^1), \hat{\mu}(t_k^3)\big).
\end{aligned}
\label{traing_property}
\end{equation}
Substituting Eqs.~\eqref{bound1}, \eqref{bound2}, \eqref{M_condition}, and
\eqref{traing_property} into Eq.~\eqref{triang}, we conclude that

\begin{equation}
  \begin{aligned}
    \bigg|W_2^2\big(\mu(t_i^1), \hat{\mu}(t_i^1)\big)& (t_{i+1}^1-t_i^1) 
    - \sum_{k=\ell}^{\ell+s-1} W_2^2\big((\mu(t_k^3),
    \hat{\mu}(t_k^3)\big)(t_{k+1}^3-t_k^3)\bigg| \\
    \: & \quad \leq 2M(t_{i+1}^1-t_i^1)\big(\sqrt{F_i\Delta t+\Sigma_i+B_i}
    + \sqrt{\hat{F}_i \Delta t+ \hat{\Sigma}_i + \hat{B}_i}\big).
    \end{aligned}
\label{intermediate}
\end{equation}
Using Eq.~\eqref{intermediate} in Eq.~\eqref{limit_exist}, when the
conditions in Eq.~\eqref{condition_dt} hold true, we have

\begin{equation}
  \begin{aligned}
    \lim_{\delta t\rightarrow0}\bigg|\sum_{i=0}^{N_1-1}W_2^2\big(\mu(t_i^1),
    \hat{\mu}(t_i^1)\big) & (t_{i+1}^1-t_i^1) -
    \sum_{i=0}^{N_3-1}W_2^2\big(\mu(t_i^3),
    \hat{\mu}(t_i^3)\big)(t_{i+1}^3-t_i^3)\bigg| \\
    \leq & 2M T\max_i\big(\sqrt{F_i\Delta t+\Sigma_i+B_i}
    + \sqrt{\hat{F}_i\Delta t + \hat{\Sigma}_i + \hat{B}_i}\big)\rightarrow0.
  \end{aligned}
  \end{equation}
Similarly,

\begin{equation}
  \begin{aligned}
    \lim_{\delta t\rightarrow0}\bigg|\sum_{i=0}^{N_2-1}W_2^2\big(\mu(t_i^2),
    \hat{\mu}(t_i^2)\big)& (t_{i+1}^2-t_i^2) -
    \sum_{i=0}^{N_3-1}W_2^2\big(\mu(t_i^3),
    \hat{\mu}(t_i^3)\big)(t_{i+1}^3-t_i^3)\bigg| \\
    \leq & 2M T\max_i\Big(\sqrt{F_i\Delta t+\Sigma_i+B_i}
    +\!\sqrt{\hat{F}_i\Delta t + \hat{\Sigma}_i + \hat{B}_i}\,\Big)\rightarrow 0.
  \end{aligned}
  \end{equation}
Thus, as $\Delta t\rightarrow 0$,

\begin{equation}
  \bigg|\sum_{i=0}^{N_1-1}W_2^2\big(\mu(t_i^1), \hat{\mu}(t_i^1)\big)(t_{i+1}^1-t_i^1)
  - \sum_{i=0}^{N_2-1}W_2^2\big(\mu(t_i^2),
  \hat{\mu}(t_i^2)\big)(t_{i+1}^2-t_i^2)\bigg|\rightarrow 0,
\label{convergence}
\end{equation}
which implies the limit

\begin{equation}
  \lim\limits_{N\rightarrow \infty} \sum_{i=0}^{N-1}\inf_{\pi_i}\E_{\pi_i}\Big[\big|\bm{X}(t_i^1)
    - \hat{\bm{X}}(t_i^1)\big|_2^2\Big](t_i^1-t_{i-1}^1)=
  \lim\limits_{N\rightarrow \infty}\sum_{i=0}^{N-1}W_2^2\big(\mu(t_i^1),
  \hat{\mu}(t_i^1)\big)(t_i^1-t_{i-1}^1)
\end{equation}
exists.  From Eq.~\eqref{tilde_def}, the limit

\begin{equation}
    \lim\limits_{N\rightarrow \infty}
    \sum_{i=1}^{N-1}\inf_{\pi_i}\E_{\pi_i}\Big[\big|\bm{X}(t_i^1) -
      \hat{\bm{X}}(t_i^1)\big|_2^2\Big](t_i^1-t_{i-1}^1)=\tilde{W}_2^2(\mu,
    \hat{\mu})
\end{equation}
Specifically, by letting
$\max_{i=0}^{n_2-1}(t_{i+1}^2-t_i^2)\rightarrow0$ in
Eq.~\eqref{convergence}, we have

\begin{equation}
  \begin{aligned}
    \bigg|\sum_{i=0}^{N_1-1}W_2^2\big(\mu(t_i^1),
    \hat{\mu}(t_i^1)\big)& (t_{i+1}^1-t_i^1)
    - \tilde{W}_2^2(\mu,\hat{\mu})\bigg| \\[-3pt]
   \leq & 2M T\max_i\Big(\sqrt{F_i\Delta t+\Sigma_i+B_i}
    + \sqrt{\hat{F}_i\Delta t + \hat{\Sigma}_i + \hat{B}_i}\,\Big).
  \end{aligned}
\end{equation}
This completes the proof of Theorem~\ref{theorem4}.

\section{Proof of Theorem~\ref{theorem5}}
\label{proof_theorem5}
Below, we provide proof for Theorem~\ref{theorem5}. First, note that

\begin{eqnarray}
  &\hspace{-2cm}\E\Big[\big|W_2^2(\mu_N^{\e}, \hat{\mu}_N^{\e})
    - W_2^2(\mu_N, \hat{\mu}_N)\big|\Big]\leq
        \E\Big[\big(W_2(\mu_N^{\e}, \hat{\mu}_N^{\e}) - W_2(\mu_N, \hat{\mu}_N)\big)^2\Big]
\nonumber\\&\hspace{5cm}\quad\quad+2\E\Big[\big|W_2(\mu_N^{\e}, \hat{\mu}_N^{\e}) - W_2(\mu_N, \hat{\mu}_N)\big|\Big] W_2(\mu_N,\hat{\mu}_N).
\label{difference}
    \end{eqnarray}
Using the triangular inequality for the Wasserstein distance
\cite{clement2008elementary}, we have

\begin{equation}
  \begin{aligned}
    \E\Big[\big|W_2(\mu_N^{\e}, \hat{\mu}_N^{\e})
      - W_2(\mu_N, \hat{\mu}_N)\big|\Big]\leq &
    \E\big[W_2(\mu_N^{\e}, \mu_N)\big]
    + \E\big[W_2(\hat{\mu}_N^{\e}, \hat{\mu}_N)\big], \\
    \E\Big[\big(W_2^2(\mu_N^{\e}, \hat{\mu}_N^{\e})
      - W_2^2(\mu_N, \hat{\mu}_N)\big)^2\Big]\leq &
    2\E\big[W_2^2(\mu_N^{\e}, \mu_N) + W_2^2(\hat{\mu}_N^{\e}, \hat{\mu}_N)\big].
\end{aligned}
\label{triag_bound}
\end{equation}
From Theorem 1 in \cite{fournier2015rate}, there exists a constant
$C_0$ depending on the dimensionality $Nd$

\begin{equation}
  \begin{aligned}
    \E\big[W_2^2(\mu_N^{\e}, \mu_N)\big] \leq &
    C_0 h^{2}(M_s, {Nd})\E\Big[\sum_{i=0}^{N-1}|\bm{X}(t_i)|_6^6
      \Delta t_i^3\Big]^{\frac{1}{3}},\\
    \E\big[W_2^2(\hat{\mu}_N^{\e}, \hat{\mu}_N)\big] \leq &
    C_0 h^{2}(M_s, {Nd})\E\Big[\sum_{i=0}^{N-1}|\hat{\bm{X}}(t_i)|_6^6
      \Delta t_i^3\Big]^{\frac{1}{3}},
\end{aligned}
\label{error_bound_M}
\end{equation}
where the function $h$ is defined in Eq.~\eqref{t_def} and $\Delta
t_i\coloneqq (t_{i+1}-t_i), i=0,...,N-1$. Substituting
Eqs.~\eqref{error_bound_M} and \eqref{triag_bound} into
Eq.~\eqref{difference}, we conclude that

\begin{equation}
\begin{aligned}
 \: &  \E\Big[\big(W_2^2 (\mu_N^{\e}, \mu_N) - W_2^2(\hat{\mu}_N^{\e},
    \hat{\mu}_N^{\e})\big)^2\Big]\\
 \: & \hspace{0.8cm} \leq 2C_0 h^{2}(M_s,{Nd}) \bigg(\E\Big[\sum_{i=0}^{N-1}\big|\bm{X}(t_i)\big|_6^6
    \Delta t_i^3\Big]^{\frac{1}{3}} +
  \E\Big[\sum_{i=0}^{N-1}\big|\hat{\bm{X}}(t_i)\big|_6^6
    \Delta t_i^3\Big]^{\frac{1}{3}}\bigg)\\
  \: & \hspace{1.5cm} +2\sqrt{C_0} W_2(\mu_N, \hat{\mu}_N)
  h(M_s,{Nd})\bigg(\E\Big[\sum_{i=0}^{N-1}\big|\bm{X}(t_i)\big|_6^6
    \Delta t_i^3\Big]^{\frac{1}{6}} + \E\Big[\sum_{i=0}^{N-1}\big|\hat{\bm{X}}(t_i)\big|_6^6
    \Delta t_i^3\Big]^{\frac{1}{6}}\bigg)
\end{aligned}
\label{ineqco}
\end{equation}
which proves the inequality~\ref{coupled_error_bound}.  Similarly, for
each $i=0,1,...,N-1$, there exists a constant $C_1$ depending on the
dimensionality $d$ such that

\begin{equation}
  \begin{aligned}
    \E\Big[\big|W_2^2(\mu^{\e}_N(t_i), &\, \hat{\mu}_N^{\e}(t_i))
      - W_2^2(\mu_N(t_i), \hat{\mu}_N(t_i))\big|\Big]\Delta t_i\\
\: & \leq 
2\sqrt{C_1}h(M_s, d)\Big(\E\Big[\big|\hat{\bm{X}}(t_i)\big|_6^6\Big]^{\frac{1}{6}}
+\E\Big[\big|\bm{X}(t_i)\big|_6^6\Big]^{\frac{1}{6}}\Big) W_2(\mu(t_i),
\hat{\mu}(t_i))\Delta t_i \\
\: & \qquad \qquad +2C_1  h^{2}(M_s, d) \Big(\E\Big[\big|\hat{\bm{X}}(t_i)\big|_6^6\Big]^{\frac{1}{3}}+
\E\Big[\big|\bm{X}(t_i)\big|_6^6\Big]^{\frac{1}{3}}\Big) \Delta t_i.
\end{aligned}
 \label{ti_errorbound0}
\end{equation}
Summing over $i$ in the inequalities~\ref{ti_errorbound0}, we find

\begin{equation}
  \begin{aligned}
   & \E\bigg[\Big|\sum_{i=0}^{N-1} \Big(W_2^2(\mu_N^{\e}(t_i), \hat{\mu}_N^{\e}(t_i)) \Delta t_i
      - W_2^2(\mu_N(t_i), \hat{\mu}_N(t_i)) \Delta t_i\Big)\Big|\bigg] \\
    &\qquad \leq\sum_{i=0}^{N-1} \E\Big[\big|W_2^2(\mu_N^{\e}(t_i), \hat{\mu}_N^{\e}(t_i))
      - W_2^2(\mu_N(t_i), \hat{\mu}_N(t_i))\big| \Delta t_i\Big] \\
\: & \qquad \leq 2\sqrt{C_1} \sum_{i=0}^{N-1}
\Bigg(\Big(\E\Big[\big|\bm{X}(t_i)\big|_6^6\Big]^{\frac{1}{6}}
+\E\Big[\big|\hat{\bm{X}}(t_i)\big|_6^6\Big]^{\frac{1}{6}}\Big)
W_2(\mu_N(t_i), \hat{\mu}_N(t_i)) \Delta t_i h(M_s, d)\\[-3pt]
\: & \hspace{4.5cm} +\sqrt{C_1}\Big(\E\Big[\big|\bm{X}(t_i)\big|_6^6\Big]^{\frac{1}{3}}
+\E\Big[\big|\hat{\bm{X}}(t_i)\big|_6^6\Big]^{\frac{1}{3}}\Big) \Delta t_i h^{2}(M_s, d)\Bigg),
\end{aligned}
\label{ineqde}
\end{equation}
which proves the inequality~\ref{decoupled_error_bound}.  Furthermore,
using the H\"older's inequality, we have

\begin{equation}
  \E\Big[\sum_{i=0}^{N-1}\big|\bm{X}(t_i)\big|_6^6
    \Delta t_i^3\Big]^{\frac{1}{3}}\cdot \E\Big[\sum_{i=0}^{N-1}1\Big]^{\frac{2}{3}}
  \geq \sum_{i=0}^{N-1}\E\Big[\big|\bm{X}(t_i)\big|_6^6\Big]^{\frac{1}{3}} \Delta t_i
\label{1stbound}
\end{equation}
and
\begin{equation}
  \E\Big[\sum_{i=0}^{N-1}\big|\hat{\bm{X}}(t_i)\big|_6^6
    \Delta t_i^3\Big]^{\frac{1}{3}} \cdot\E\Big[\sum_{i=0}^{N-1}1\Big]^{\frac{2}{3}}
  \geq \sum_{i=0}^{N-1}\E\Big[\big|\hat{\bm{X}}(t_i)\big|_6^6\Big]^{\frac{1}{3}} \Delta t_i
\label{2ndbound}
\end{equation}
Furthermore, for any coupled distribution $\pi(\bm{X}_N,
\hat{\bm{X}}_N)$ whose marginal distributions are $\mu_N$ and
$\hat{\mu}_N$, we have, by using the Cauchy inequality,

\begin{equation}
  \begin{aligned}
    2\sum_{i=0}^{N-1} & 
    \Big(\E\Big[\big|\hat{\bm{X}}(t_i)\big|_6^6\Delta t_i^3\Big]^{\frac{1}{6}}
    +\E\Big[\big|\hat{\bm{X}}(t_i)\big|_6^6\Delta t_i^3\Big]^{\frac{1}{6}}\Big)\cdot
    \E_{\pi(\bm{X}_N, \hat{\bm{X}}_N)}\Big[\big|\hat{\bm{X}}(t_i)
      - \bm{X}(t_i)|_2^2\Delta t_i\Big]^{\frac{1}{2}} \\
    \: \quad & \leq 2\E\Big[\sum_{i=0}^{N-1}1\Big]^{\frac{1}{3}}\cdot
    \bigg(\E\Big[\sum_{i=0}^{N-1}\big|\hat{\bm{X}}(t_i)\big|_6^6
      \Delta t_i^{3} \Big]^{\frac{1}{6}}  +\E\Big[\sum_{i=0}^{N-1}\big|\hat{\bm{X}}(t_i)\big|_6^6
      \Delta t_i^{3}\Big]^{\frac{1}{6}}\bigg)\\[-2pt]
    \: & \hspace{4.5cm}\times
    \E_{\pi(\bm{X}_N, \hat{\bm{X}}_N)}\Big[\sum_{i=0}^{N-1}\big|\hat{\bm{X}}(t_i)
      - \bm{X}(t_i)\big|_2^2\Delta t_i\Big]^{\frac{1}{2}}
  \end{aligned}
\end{equation}
Therefore, by taking the infimum over all coupling distributions
$\pi(\bm{X}_N, \hat{\bm{X}}_N)$, we conclude that

\begin{equation}
  \begin{aligned}
    2\sum_{i=0}^{N-1}& \Big(\E\Big[\big|\hat{\bm{X}}(t_i)\big|_6^6\Delta t_i^3\Big]^{\frac{1}{6}}
   + \E\Big[\big|\hat{\bm{X}}(t_i)\big|_{6}^6\Delta t_i^3\Big]^{\frac{1}{6}}\Big) W_2(\mu(t_i),
    \hat{\mu}(t_i)) \sqrt{\Delta t_i} \\
\: & \leq 2\sum_{i=0}^{N-1} \Big(\E\Big[\big|\hat{\bm{X}}(t_i)\big|_6^6\Delta t_i^3\Big]^{\frac{1}{6}}
+\E\Big[\big|\hat{\bm{X}}(t_i)\big|^6_6\Delta t_i^3\Big]^{\frac{1}{6}}\Big)
\inf_{\pi(\bm{X}_N, \hat{\bm{X}}_N)}\!\!\E_{\pi(\mu_N, \hat{\mu}_N)}
 \Big[\big|\hat{\bm{X}}(t_i)
   - \bm{X}(t_i)\big|_2^2\Delta t_i\Big]^{\frac{1}{2}}\\
 \: & \quad \leq \bigg(\E\Big[\sum_{i=0}^{N-1}\big|\hat{\bm{X}}(t_i)\big|_6^6
   \Delta t_i^{3}\Big]^{\frac{1}{6}}
 +\E\Big[\sum_{i=0}^{N-1}\big|\hat{\bm{X}}(t_i)\big|_6^6
   \Delta t_i^{3}\Big]^{\frac{1}{6}}\bigg)W_2(\mu_N, \hat{\mu}_N) N^{\frac{1}{3}}.
 \end{aligned}
  \label{generation_bound}
\end{equation}
After combining the five inequalities Eqs.~\eqref{1stbound},
\eqref{2ndbound}, ~\eqref{ineqco}, \eqref{ineqde}, and
\eqref{generation_bound}, we conclude that

\begin{equation}
     E_1(M_s)\geq C E_2(M_s) \frac{h(M_s, Nd)}{h(M_s, d)}N^{-\frac{2}{3}},
\end{equation}
where $C\coloneqq\min\Big\{\sqrt{\frac{C_0}{C_1}},
\frac{C_0}{C_1}\Big\}\cdot \max_{x\geq 1}\frac{h(x, 4)}{h(x, 5)}\leq
20\min\Big\{\sqrt{\frac{C_0}{C_1}}, \frac{C_0}{C_1}\Big\}$.

\section{Default training settings}
\label{training_details}
We list the training hyperparameters and gradient descent
methods for each example in Table~\ref{tab:setting}.
\begin{table}[h!]
\centering
\caption{Training settings for each example.} 
\begin{tabular}{lllll}
\toprule {Loss} & Example \ref{example1} & Example \ref{example2} &
Example \ref{example3} \\
\midrule
Gradient descent method & AdamW & AdamW & AdamW \\
Learning rate & 0.002 & 0.003 & 0.002 \\
Weight decay & 0.005 & 0.02 & 0.005 \\
No. of epochs & 1000 & 500 & 400 \\
No. of training trajectories $M_s$ & 100 & 400 & 300 \\
Hidden layers in $\Theta_1$ & 2 & 2 & $\backslash$ \\
Hidden layers in $\Theta_2$ & 2 & 2 & 3 \\
Hidden layers in $\Theta_3$ & 2 & 2 & 3 \\
Activation function & ReLu & ReLu & ReLu \\
Neurons in each layer in $\Theta_1$ & 150 & 150 &
$\backslash$ \\
Neurons in each layer in $\Theta_2$ & 150 & 150 & 400 \\
Neurons in each layer in $\Theta_3$ & 150 & 150 & 400 \\
$\Delta t$ & 0.2 & 0.1 & 0.2 \\
Number of timesteps $N$ &101 & 51 & 51\\
Initialization & \texttt{torch.nn} default & \texttt{torch.nn} default
& \multirow{2}{*}{} 0 for biases  \\ & & & $\mathcal{N}(0, 10^{-4})$
  for weights\\
repeat times & 10 & 5 & 5\\
\bottomrule
\end{tabular}
\label{tab:setting}
\end{table}

\section{Definitions of different loss metrics}
\label{def_loss}
Here, we provide definitions of loss functions used in our numerical
examples (the definitions of the MSE, mean$^2$+var, and the MMD loss
functions are the same as Appendix~E in \cite{xia2023a}). Since we are
using a uniform mesh grid ($t_{i+1}-t_i=\Delta t, \forall
i=0,...,N-1$), for simplicity, we shall omit $\Delta t$ in the
calculation of our loss functions:

\begin{compactenum}
\item The squared Wasserstein-2 distance
$$W_2^2(\mu_N, \hat{\mu}_N)\approx W_2^2(\mu_N^{\text{e}},
  \hat{\mu}_N^{\text{e}}),$$ where $\mu_N^{\text{e}}$ and
  $\hat{\mu}_N^{\text{e}}$ are the empirical distributions of the
  vector $(\bm{X}(t_0),...\bm{X}(t_{N-1}))$ and
  $(\hat{\bm{X}}(t_0),...,\hat{\bm{X}}(t_{N-1}))$, respectively. In
  numerical examples, we use the following scaled squared
  Wasserstein-2 distance:
\begin{equation}
\frac{1}{\Delta t}W_2^2(\mu_N^{\text{e}},
\hat{\mu}_N^{\text{e}})\approx\texttt{ot.emd2}\bigg(\frac{1}{M_s}\bm{I}_{M_s},
\frac{1}{M_s}\bm{I}_{M_s}, \bm{C}\bigg),
\label{time_coupling}
\end{equation}
where $\texttt{ot.emd2}$ is the function for solving the earth movers
distance problem in the $\texttt{ot}$ package of Python
\cite{flamary2021pot}, $M_s$ is the number of ground truth and
predicted trajectories, $\bm{I}_{\ell}$ is an $M_s$-dimensional vector
whose elements are all 1, and $\bm{C}\in\mathbb{R}^{M_s\times M_s}$ is
a matrix with entries $(\bm{C})_{ij} =
|\bm{X}_N^i-\hat{\bm{X}}^j_N|_2^2$. $|\cdot|_2$ is the 2-norm of a vector. $\bm{X}^i
_N$ is the vector of the values of the $i^{\text{th}}$ ground truth
trajectory at time points $t_0,...,t_{N-1}$, and $\hat{\bm{X}}^j _N$ is the
vector of the values of the $j^{\text{th}}$ predicted trajectory at
time points $t_0,...,t_{N-1}$.
\item The temporally decoupled squared Wasserstein-2 distance
  (Eq.~\eqref{temporal_calculation}). In numerical examples, we use the
  following scaled temporally decoupled squared Wasserstein-2
  distance:
  $$\frac{1}{\Delta t}\tilde{W}^2_2(\mu_N, \hat{\mu}_N) \approx \sum_{i=1}^{N-1}
  W^2_2(\mu^{\text{e}}(t_i),
  \hat{\mu}^{\text{e}}(t_i)),$$
where $\Delta t$ is the time step and $W_2$ is the Wasserstein-2
distance between two empirical distributions of $\bm{X}(t_i)$ and $\hat{\bm{X}}(t_i)$, denoted by $\mu^{\text{e}}(t_i),
\hat{\mu}^{\text{e}}(t_i)$, respectively.  These distributions are calculated by
the samples of the trajectories of $\bm{X}(t), \hat{\bm{X}}(t)$ at a given time
step $t=t_i$, respectively. $W^2_2(\mu_{N}^{\text{e}}(t_i),
\hat{\mu}^{\text{e}}_{N}(t_i))$ is calculated using the
$\texttt{ot.emd2}$ function, \textit{i.e.},

\begin{equation}
W_2^2(\mu^{\text{e}}(t_i),
\hat{\mu}^{\text{e}}(t_i))\approx\texttt{ot.emd2}\Big(\frac{1}{M_s}\bm{I}_{M_s},
\frac{1}{M_s}\bm{I}_{M_s}, \bm{C}(t_i)\Big),
\label{time_coupling}
\end{equation}
where $\bm{I}_{\ell}$ is an $M$-dimensional vector whose elements are
all 1, and $\bm{C}\in\mathbb{R}^{M_s\times M_s}$ is a matrix with
entries $(\bm{C})_{sj} = |\bm{X}^s(t_i)-\hat{\bm{X}}^j(t_i)|_2^2$. $\bm{X}^s(t_i)$ is the vector of values of the $s^{\text{th}}$ ground truth
trajectory at time $t_i$ and $\hat{\bm{X}}^s(t_i)$ is the vector of values of the $s^{\text{th}}$ 
trajectory generated by the reconstructed jump-diffusion process at time $t_i$.
\item The Wasserstein-1 distance
$$W_1(\mu_N, \hat{\mu}_N)\approx W_1(\mu_N^{\text{e}},
  \hat{\mu}_N^{\text{e}}),$$ where $\mu_N^{\text{e}}$ and
  $\hat{\mu}_N^{\text{e}}$ are the empirical distributions of the
  vector $(\bm{X}(t_0),...\bm{X}(t_{N-1}))$ and
  $(\hat{\bm{X}}(t_0),...,\hat{\bm{X}}(t_{N-1}))$, respectively. In
  numerical examples, we use the following scaled $W_1$ distance:
\begin{equation}
\frac{1}{\Delta t}W_1(\mu_N^{\text{e}},
\hat{\mu}_N^{\text{e}})\approx\texttt{ot.emd2}\bigg(\frac{1}{M_s}\bm{I}_{M_s},
\frac{1}{M_s}\bm{I}_{M_s}, \bm{C}\bigg),
\label{time_coupling}
\end{equation}
where $\texttt{ot.emd2}$ is the function for solving the earth movers
distance problem in the $\texttt{ot}$ package of Python, $M_s$ is the
number of ground truth and predicted trajectories, $\bm{I}_{\ell}$ is
an $M_s$-dimensional vector whose elements are all 1, and
$\bm{C}\in\mathbb{R}^{M_s\times M_s}$ is a matrix with entries
$(\bm{C})_{ij} = |\bm{X}_N^i-\hat{\bm{X}}^j_N|_2$. $\bm{X}^i_N$ is the vector of the
values of the $i^{\text{th}}$ ground truth trajectory at time points
$t_0,...,t_{N-1}$, and $\hat{\bm{X}}^j _N$ is the vector of the values of
the $j^{\text{th}}$ predicted trajectory at time points
$t_0,...,t_{N-1}$.
\item Mean squared error (MSE) between the trajectories, where $M_s$
  is the total number of the ground truth and predicted
  trajectories. $\bm{X}_{i, j}$ and $\hat{\bm{X}}_{i, j}$ are the values of the
  $j^{\text{th}}$ ground truth and prediction trajectories at time
  $t_i$, respectively:
$$\operatorname{MSE}(\bm{X}, \widehat{\bm{X}}) =
  \frac{1}{M_SN}\sum_{i=0}^{N-1} \sum_{j=1}^{M_s} (\bm{X}_{i, j}-\hat{\bm{X}}_{i,
    j})^2.$$
\item The summation of squared distance between mean trajectories and
  absolute values of the discrepancies in variances of trajectories,
  which is a common practice for estimating the parameters of an SDE.
  We shall denote this loss function by
  \begin{eqnarray*}
    (\operatorname{mean}^2+\operatorname{var})({ \bm{X}}, \hat{\bm{X}}) =
  \sum_{i=0}^{N-1} \left[\bigg(\frac{1}{M_s}\sum_{j=1}^{M_s}
    \big(\bm{X}_{i, j} - \hat{\bm{X}}_{i, j}\big)\bigg)^{2} + \left|
    \operatorname{var}(\bm{X}_i) - \operatorname{var}(\hat{\bm{X}}_i)
    \right|\right].
\end{eqnarray*}
Here $M_s$ and $\bm{X}_{i, j}$ and $\hat{\bm{X}}_{i, j}$ have the same meaning
as in the MSE definition. $\operatorname{var}(\bm{X}_i)$ and
$\operatorname{var} (\hat{\bm{X}}_i)$ are the variances of the empirical
distributions of $\bm{X}(t_i), \hat{\bm{X}}(t_i)$, respectively.
\item MMD (maximum mean discrepancy) In our numerical examples, we use
  the following MMD loss function \cite{li2015generative}:
$$
\text{MMD}(X, \hat{X}) = \sum_{i=1}^{N-1}\Big(\E\big[K(\bm{X}_i, \bm{X}_i)\big]
- 2\E\big[K(\bm{X}_i, \hat{\bm{X}}_i)\big] + \E\big[K(\hat{\bm{X}}_i, \hat{\bm{X}}_i)\big]\Big),
$$
where $K$ is the standard radial basis function (or Gaussian kernel)
with multiplier $2$ and number of kernels $5$. $\bm{X}_i$ and $\hat{\bm{X}}_i$
are the values of the ground truth and predicted trajectories at time
$t_i$, respectively.

\end{compactenum}

\section{Varying the coefficients that determine diffusion and jump functions}
\label{noise_strength}
Here, we consider changing the two parameters $\sigma_0, y_0$ in
Eq.~\eqref{example1_numerical} of Example~\ref{example1}. With larger
$\sigma_0, y_0$, the trajectories generated by
Eq.~\eqref{example1_numerical} will be subject to greater
fluctuations. We use the temporally squared $W_2$ distance as the loss
function. We vary $\sigma_0$ to range from 0.2 to 0.4 and vary $y_0$
from 0.5 to 1. We repeat our experiments 10 times, and we plot the
temporally squared $W_2$ distance as well as the errors of the
reconstructed $\hat{f}, \hat{\sigma}, \hat{\beta}$.
\begin{figure}[h!]
\centering
\includegraphics[width=\linewidth]{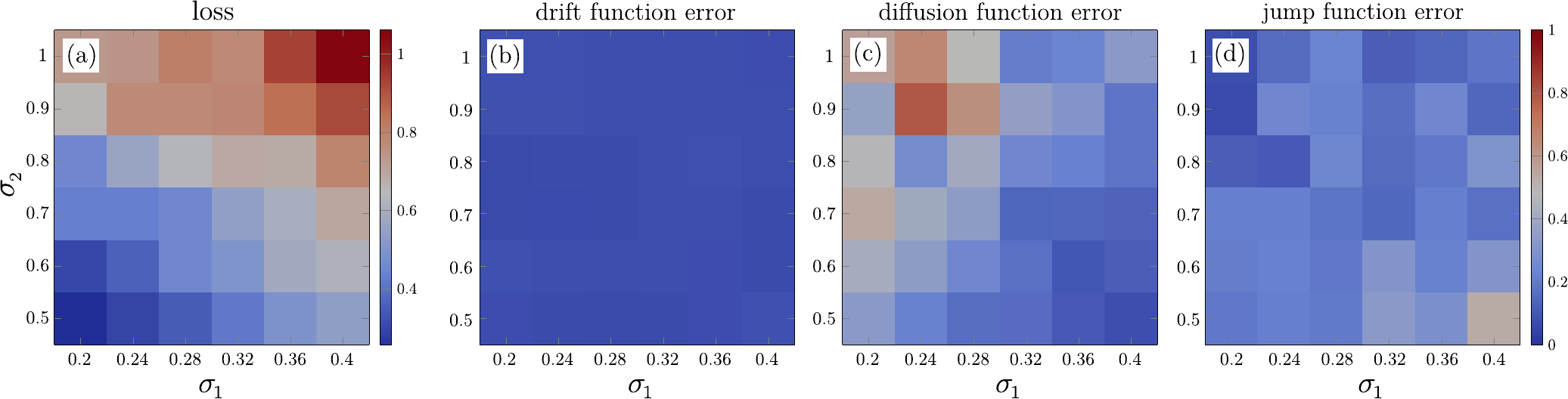}
\caption{(a) The temporally decoupled squared Wasserstein distance
$\tilde{W}_2^2(\mu_N, \hat{\mu}_N)$. (b) the average relative
errors in the reconstructed drift function $\hat{f}(x)$; (c) the
average relative errors in the reconstructed diffusion function
$\hat{\sigma}(x)$; (d) the average relative errors in the
reconstructed jump functions $\hat{\beta}(x)$.}
\label{fig:example1:noise_strength}
\end{figure}
From Fig.~\ref{fig:example1:noise_strength}(a), larger $\sigma_0, y_0$
lead to larger $\tilde{W}_2^2(\mu_N, \hat{\mu}_N)$. This could be
because larger $\sigma_0, y_0$ lead to ground truth trajectories with
larger fluctuations, rendering the underlying dynamics harder to
reconstruct. Fig.~\ref{fig:example1:noise_strength}(b) implies that
the drift function can be accurately reconstructed and is insensitive
to different $\sigma_0, y_0$.  As seen in
Fig.~\ref{fig:example1:noise_strength}(c), if $\sigma_0$ is small, the
relative error in the reconstructed diffusion function can be well
controlled around $0.1$; when $\sigma_0$ is larger, it is harder to
reconstruct the diffusion function and the relative error in the
reconstructed diffusion function $\hat{\sigma}$ will be
larger. Fig.~\ref{fig:example1:noise_strength}(d) shows that the
reconstruction of the jump function $\hat{\beta}$ is not very
sensitive to different values of $\sigma_0$ and $y_0$.

\section{Reconstructing Eq.~\eqref{example2_model}
in Example~\ref{example2} with different numbers of trajectories in
the training set}
\label{num_traj}

Here, we carry out an additional numerical experiment of
reconstructing Eq.~\eqref{example2_model} by changing the number of
trajectories in the training set. We define the ground truth
jump-diffusion process by the drift function $\alpha(X, t)\coloneqq
r$, and the diffusion function and jump functions $\sigma(X,
t)=\beta(X, t)=0.1\sqrt{|X|}$.  We consider four scenarios: i)
provide no prior information and reconstruct drift, diffusion, and
jump functions, ii) provide the drift function $\alpha(X, t)$ as
prior information and reconstruct the diffusion and jump functions,
iii) provide the diffusion function $\sigma(X, t)$ as prior
information and reconstruct the drift and jump functions, and iv)
provide the jump function $\beta(X, t)$ as prior information and
reconstruct the drift and diffusion functions.

\begin{figure}[h!]
   \centering
    \includegraphics[width=0.92\linewidth]{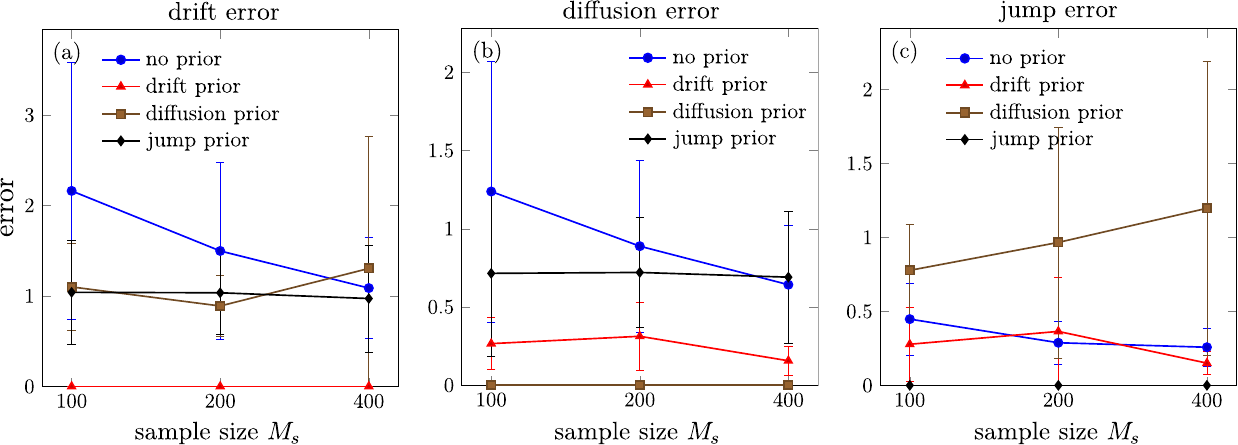}
    \caption{The reconstruction errors in the drift, diffusion, and
      jump functions defined in Eqs. \ref{diffusion_error} and
      \ref{jump_error} as a function of the number of trajectories
      $M_s$ when different prior information is provided.  The results
      are averaged over 5 independent experiments. Training
      hyperparameters are the same as those used in
      Example~\ref{example2} listed in Table~\ref{tab:setting}.}
    \label{fig:example2_appendix3}
\end{figure}

As seen in Figs.~\ref{fig:example2_appendix3}(b-c), providing the
drift function as prior information greatly boosts the efficiency of
our temporally decoupled squared $W_2$ method allowing it to
accurately reconstructing the unknown diffusion and jump functions
even with as few as 100 trajectories for training.  Also, even with no
prior information, the errors in the reconstructed drift, diffusion,
or jump function decrease when the number of trajectories in the
training set increases (Figs.~\ref{fig:example2_appendix3}(a-c)). This
indicates that even without prior information, our temporally
decoupled squared $W_2$ method can accurately reconstruct
Eq.~\eqref{example2_model} when provided a sufficient number of training
trajectories.

When the diffusion or jump function is given as prior information, the
errors of the reconstructed unknown functions do not decrease much as
the number of trajectories for training $M_s$ increases.  Even with the
correct diffusion or jump function, different realizations of the
Brownian motion or the compensated Poisson process yield very different
trajectories so that providing the diffusion or jump function may
provide little information in discriminating trajectories.


\section{Reconstructing Eq.~\eqref{example2_model}
  in Example~\ref{example2} with different parameters in the diffusion
  and jump functions when providing the drift function}
\label{varying_coef}
Here, given the drift function $\alpha(X, t)\coloneqq r$, we carry
out an additional numerical experiment of reconstructing
Eq.~\eqref{example2_model} by varying the parameters $\sigma_0, \beta_0$
that determine the strength of the Brownian-type and
compensated-Poisson-type noise in Eqs.~\eqref{constant_2},
~\eqref{linear_2}, and ~\eqref{langevin_2}.

\begin{figure}[h!]
    \centering
    \includegraphics[width=\linewidth]{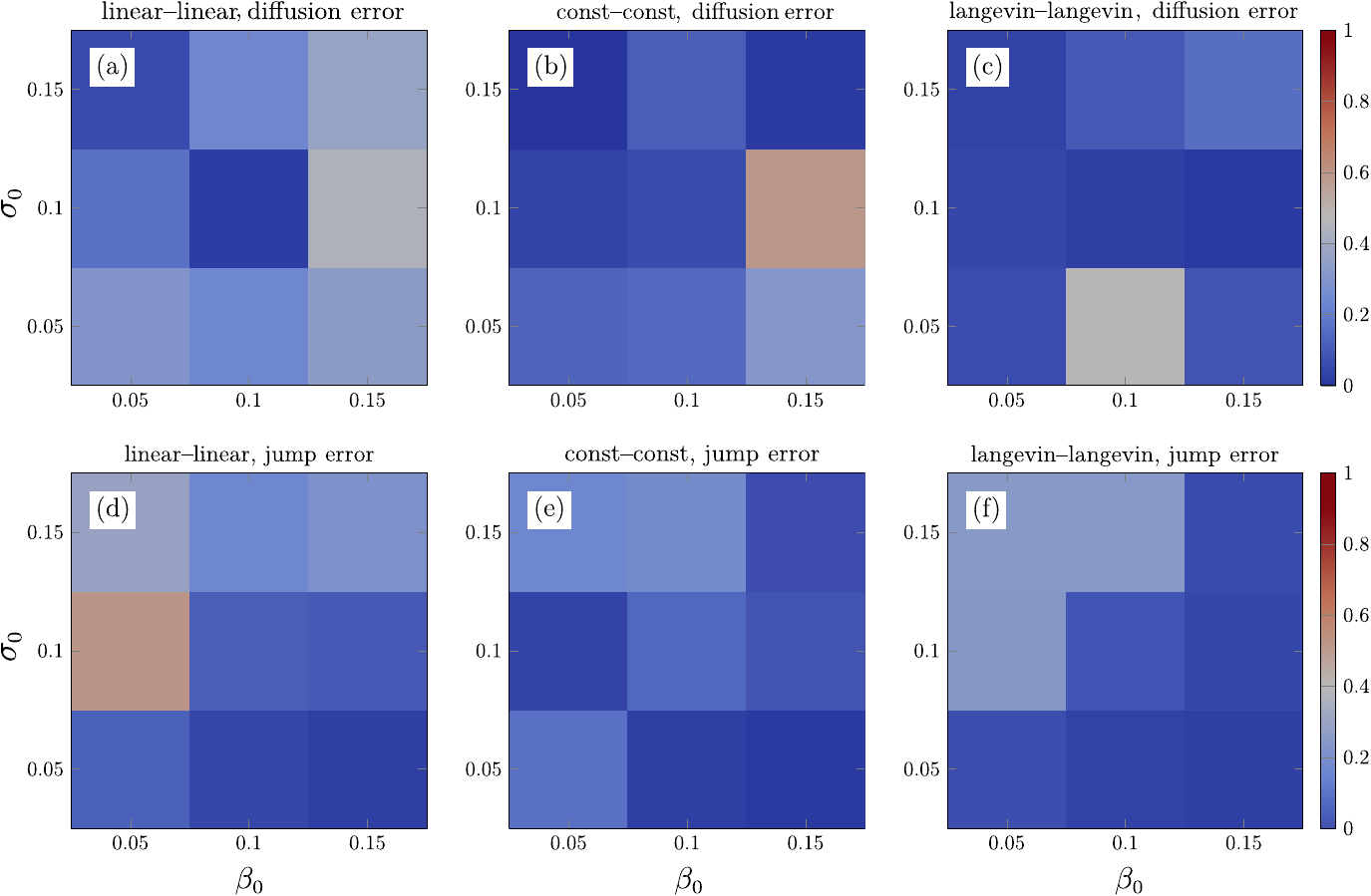}
    \caption{The reconstruction errors in the diffusion, and jump
      functions defined in Eqs. \ref{diffusion_error} and
      \ref{jump_error} w.r.t. the two parameters that determine the
      strength of noise $\sigma_0$ and $\beta_0$ in
      Eqs.~\eqref{constant_2}, ~\eqref{linear_2}, and
      ~\eqref{langevin_2}. Here, const-const indicates we are using
      Eq.~\eqref{constant_2} for both the diffusion and jump functions;
      linear-linear indicates we are using Eq.~\eqref{linear_2} for both
      the diffusion and jump functions; langevin-langevin indicates we
      are using Eq.~\eqref{langevin_2} for both the diffusion and jump
      functions. The results are averaged over 5 independent
      experiments. Training hyperparameters are the same as
      Example~\ref{example2} in Table~\ref{tab:setting}.}
    \label{fig:example2_appendix2}
\end{figure}

Fig.~\ref{fig:example2_appendix2} shows our temporally decoupled
squared $W_2$-distance loss function can be used to accurately
reconstruct the diffusion function and the jump function $\sigma(X,
t), \beta(X, t)$ in Eq.~\eqref{example2_model},  even when different
forms of $\sigma(X, t), \beta(X, t)$ in Eqs.~\eqref{constant_2},
~\eqref{linear_2}, and ~\eqref{langevin_2} and different noise strengths
$\sigma_0, \beta_0$ are given. The average errors (averaged over all
choices of $\sigma_0, \beta_0$) in the reconstructed diffusion function
$\hat{\sigma}$ are 0.171 (const-const), 0.217 (linear-linear), and
0.176 (langevin-langevin).  The average errors (averaged over all
choices of $\sigma_0, \beta_0$) in the reconstructed jump function
$\hat{\sigma}$ are 0.173 (const-const), 0.188 (linear-linear), and
0.184 (langevin-langevin).



\section{Neural Network Architecture}
\label{number_layers}

Here, we investigate how the neural network architecture,
\textit{i.e.}, the number of hidden layers and the number of neurons
in each layer, influence the accuracy of reconstructing the 2D
jump-diffusion process (Eq.~\eqref{example3_model}).  We vary only the
number of hidden layers and the number of neurons per layer for the
parameterized neural networks that we use to approximate the diffusion
and jump functions $\bm{\sigma}$ and $\bm{\beta}$ in
Eq.~\eqref{example3_model}. We set the parameters to be $c_1=-0.5,
c_2=-1$ and $\sigma_0=\beta_0=0.1$ in Eqs.~\eqref{sigma2d}
and~\eqref{beta2d}, and consider 200 trajectories.

\begin{table}[h!]
\centering
  \caption{Reconstructing the jump-diffusion process
    Eq.~\eqref{example3_model} when using neural networks with different
    numbers of hidden layers and neurons per layer to parameterize
    $\hat{f}, \hat{\sigma}$. Other training hyperparameters are the
    same as those used in Table~\ref{tab:setting} of
    Example~\ref{example3}.}
\begin{tabular}{rccllc}
\toprule  Width & Layer & Relative Errors in $\hat{\bm{\sigma}}$ & Relative
Errors in $\hat{\bm{\beta}}$ & $N_{\rm repeats}$ \\ 
\midrule 
  25 & 3 & \(0.6836(\pm 0.5177) \) & \( 0.5554 (\pm 0.4024) \) & 5 \\ 
  50 & 3 & \( 0.8051 (\pm 0.4756) \) & \( 0.7413 (\pm 0.3515) \) & 5 \\ 
100 & 3 & \( 0.6376 (\pm 0.3261) \) & \( 0.5085 (\pm 0.2841) \) & 5 \\ 
200 & 3 & \( 0.6101 (\pm 0.2435) \) & \( 0.5280 (\pm 0.2038) \) & 5\\ 
400 & 3 & \( 0.2619 (\pm 0.1859) \) & \( 0.2837 (\pm 0.1961) \) & 5\\ 
 200 & 1 & \( 0.7143 (\pm 0.8451) \) & \( 0.6178 (\pm 0.2925) \)
& 5 \\ 
200 & 2 & \( 0.6984 (\pm 0.4989) \) & \( 0.6326 (\pm
0.4445) \) & 5 \\ 
 200 & 4 & \( 0.7605\pm 0.3837) \) & \( 0.6750\pm 0.2761) \)  & 5\\
\bottomrule
\end{tabular}
\label{tab:example3}
\end{table}

From Table~\ref{tab:example3}, we see that increasing the number of
hidden layers and increasing the number of neurons per hidden layer
can both increase the accuracy of the reconstructed
$\hat{\bm{\sigma}}$ and $\hat{\bm{\beta}}$. However, with a fixed
number of neurons per hidden layer (200), when the number of hidden
layers in the feed-forward neural network is greater than 3, the
errors in the reconstructed $\bm{\sigma}$ and $\bm{\beta}$
increase. This behavior may be due to vanishing gradients during
training of deep neural networks \cite{glorot2010understanding}; in
this case, the ResNet technique \cite{he2016deep} can be considered if
deep neural networks are used.  On the other hand, using a deeper or
wider network requires more memory usage and longer run times. For
reconstructing Eq.~\eqref{example3_model}, we find an optimal neural
network architecture consisting of about three hidden layers
containing $\sim 400$ neurons each. How optimal architectures evolve
when reconstructing different multidimensional jump-diffusion
processes requires further exploration.

\end{document}